\documentclass[accepted]{uai2025} % for initial submission
\usepackage[american]{babel}
\usepackage{natbib} % has a nice set of citation styles and commands
\usepackage{mathtools} % amsmath with fixes and additions
\usepackage{booktabs} % commands to create good-looking tables
\usepackage{tikz} % nice language for creating drawings and diagrams

\title{What is the Right Notion of Distance between Predict-then-Optimize Tasks?}

\author[1]{Paula Rodriguez-Diaz}
\author[1]{Lingkai Kong}
\author[2]{Kai Wang}
\author[1,3,4]{David Alvarez-Melis}
\author[1]{Milind Tambe}
\affil[1]{%
    Harvard University
}
\affil[2]{%
    Georgia Institute of Technology
}
\affil[3]{%
    Kempner Institute
}
\affil[4]{%
    Microsoft Research
}
  
\usepackage{graphicx} % Required for inserting images
\usepackage{float}
\usepackage{natbib}
\usepackage{xcolor}
\usepackage{amsmath}
\usepackage{amssymb}
\usepackage{bm}
\usepackage{enumitem}
\usepackage{soul}
\usepackage{booktabs}
\usepackage{physics}
\usepackage{catchfilebetweentags}
\usepackage{algorithm}
\usepackage{algpseudocode}
\usepackage{subcaption}
\usepackage{caption}

\definecolor{darkblue}{rgb}{0.0, 0.0, 0.5}
\hypersetup{
    colorlinks=true,
    allcolors=darkblue
}

\newcommand{\R}{\mathbb{R}}
\newcommand{\E}{\mathbb{E}}
\newcommand{\Et}{\mathbb{E}_{(x,y) \sim \Pd_T} \ }
\newcommand{\Etf}{\mathbb{E}_{(x,y,z) \sim \Pd_{T}^{f}} \ }
\newcommand{\Etstar}{\mathbb{E}_{(x,y,z) \sim \Pd_{S}^{*}} \ }

\newcommand{\D}{\mathcal{D}}
\newcommand{\X}{\mathcal{X}}
\newcommand{\Y}{\mathcal{Y}}
\newcommand{\Z}{\mathcal{W}}
\newcommand{\C}{\mathcal{Y}}
\newcommand{\Pd}{\mathcal{P}}

\newcommand*\dt{\mathop{}\!\mathrm{d}}
\newcommand{\dtproof}{\dt\Pi^*((x_s,y_s,z_s),(x_t,y_t,z_{t}^{f}))}
\newcommand{\dtproofshort}{\dt\Pi^*(\vb{w}_{s},\vb{w}_{t}^{f})}

\newcommand{\dq}{q}
\newcommand{\dqreg}{\dq_{\text{reg}}}
\newcommand{\ldq}{l_{g}}

\newcommand{\argmax}{\text{\normalfont argmax}}

\usepackage{amsmath}
\usepackage{amssymb}
\usepackage{mathtools}
\usepackage{amsthm}

\theoremstyle{plain}
\newtheorem{theorem}{Theorem}[section]
\newtheorem{proposition}[theorem]{Proposition}
\newtheorem{lemma}[theorem]{Lemma}

\theoremstyle{definition}
\newtheorem{definition}[theorem]{Definition}
\newtheorem{assumption}[theorem]{Assumption}
\theoremstyle{remark}

\begin{document}
\maketitle

\begin{abstract}
Comparing datasets is a fundamental task in machine learning, essential for various learning paradigms—from evaluating train and test datasets for model generalization to using dataset similarity for detecting data drift. While traditional notions of dataset distances offer principled measures of similarity, their utility has largely been assessed through prediction error minimization. However, in Predict-then-Optimize (PtO) frameworks, where predictions serve as inputs for downstream optimization tasks, model performance is measured through decision regret rather than prediction error. In this work, we propose OTD$^3$ (Optimal Transport Decision-aware Dataset Distance), a novel dataset distance that incorporates downstream decisions in addition to features and labels. We show that traditional feature-label distances lack informativeness in PtO settings, while OTD$^3$ more effectively captures adaptation success. We also derive a PtO-specific adaptation bound based on this distance. Empirically, we show that our proposed distance accurately predicts model transferability across three different PtO tasks from the literature. Code is available at \url{https://github.com/paularodr/OTD3}
\end{abstract}

\section{Introduction}
%importance of comparing datasets: comparing datasets to asses their similarity is an underlying problem serving multiple tasks

Comparing datasets is a fundamental task in machine learning and a crucial component of various downstream tasks. Understanding the \textit{similarity} (or \textit{dissimilarity}) of datasets can inform decisions in transfer learning \citep{tran_transferability_2019,ben-david_theory_2010}, multitask learning \citep{janati_wasserstein_2019,shui_principled_2019}, and data valuation \citep{just_lava_2023,jiang_opendataval_2023}, among other applications. For example, selecting a pre-training dataset that is similar to a data-poor target domain can lead to better fine-tuning performance. Notions of \textit{dataset distance} have emerged as a principled way of quantifying these similarities and differences \citep{mercioni_survey_2019, janati_wasserstein_2019, alvarez-melis_geometric_2020}. Such distances provide insights into the relation and correspondence between data distributions, help in evaluating model performance, and guide the selection of appropriate learning algorithms. 

% datasets have different structues. only feature, or features and labels.
% highlight how dataset distance has been used to compare such datasets

The concept of \textit{dataset} can vary based on context and objectives. In classical statistics, it generally refers to feature vectors, focusing on the distribution and relationships within a feature space $\X$. 
% Traditional statistical tests, such as the chi-squared test for categorical variables \citep{pearson_criterion_1900} and the Kolmogorov-Smirnov test for numerical variables \citep{massey_kolmogorov-smirnov_1951}, quantify similarity based on features alone. Additionally, 
Classic distributional distances offer formal measures of dataset similarity: the Total Variation distance \citep{verdu_total_2014} quantifies the maximum discrepancy between distributions; Wasserstein distance, or Earth Mover's Distance, measures the cost of transforming one distribution into another \citep{villani_optimal_2008}; and Integral Probability Metrics (IPM) measure how well a class of classifiers can distinguish samples from the two distributions \citep{muller_integral_1997}.

In supervised learning, datasets include both features from space $\X$ and labels from space $\Y$. The distance between two such datasets involves measuring both the feature and label differences. This can be challenging when the label space $\Y$ is not a metric space. Approaches such as those proposed by  \citet{courty_domain_2014}, \citet{alvarez-melis_structured_2018}, and \citet{alvarez-melis_geometric_2020} offer a principled method for computing dataset distances considering the joint feature-label distribution $\Pd(\X \times \Y)$. These methods ensure that both the features and labels are adequately accounted for in the distance measure, offering a more holistic comparison between datasets. \looseness=-1

% PtO datasets are unique and different to the datasets that have been considered in past dataset distance literature
However, the Predict-then-Optimize (PtO) framework introduces a unique challenge by using machine learning predictions as inputs for a downstream optimization problem, shifting the focus from minimizing prediction error to minimizing decision regret \citep{donti_task-based_2017, elmachtoub_smart_2022, wilder_melding_2019, mandi_decision-focused_2023}. This results in PtO tasks involving not just a feature-label dataset, but also a decision space $\Omega$ of optimization solutions, creating a feature-label-decision dataset with samples in $\X \times \Y \times \Omega$.
The decision space $\Omega$ may not be a metric space; for example, decisions related to the solution of a top-k problem do not necessarily form a metric space. Moreover, decisions might need to be evaluated under various criteria, such as minimizing travel distance or maximizing safety. Even if $\Omega$ were a metric space, it is uncertain whether its associated distance would be meaningful for assessing the adaptability of a PtO task across different domains. This complexity underscores the need for new distance measures that incorporate decisions to accurately capture the nature of PtO tasks. \looseness=-1

%what we do in this work
In this work we introduce OTD$^3$, a decision-aware dataset distance based on Optimal Transport (OT) techniques \citep{villani_optimal_2008} that incorporates features, labels, and decisions. OTD$^3$ is the first distance metric designed to account for downstream decisions, directly addressing the unique challenges posed by PtO tasks. We evaluate its utility as a learning-free criterion for assessing the transferability of models trained within a \textit{decision-focused learning} (DFL) framework under distribution shift. Within this framework, models are commonly trained using surrogate loss functions that aim to minimize decision regret \citep{wilder_melding_2019,mandi_decision-focused_2023}. In the context of domain adaptation in PtO, where we measure performance through decision regret, we derive a generalization bound that highlights the importance of considering features, labels, and decisions jointly. Our empirical analysis spans three PtO tasks from the literature—Linear Model Top-K, Warcraft Shortest Path, and Inventory Stock Problem—demonstrating that our decision-aware distance better predicts transfer performance compared to feature-label distances alone.

In summary, we make the following contributions:\vspace{-5pt}
\begin{itemize}[leftmargin=0.5cm, noitemsep, topsep=0pt]
\item We introduce a decision-aware dataset distance 
\item We derive an adaptation bound for PtO tasks in terms of this distance
\item We empirically validate our approach on three PtO tasks
\end{itemize}

%=========================
\section{Related Work}

\paragraph{Dataset Distances via Optimal Transport}

Optimal Transport (OT)-based distances have gained traction as an effective method for comparing datasets. These methods characterize datasets as empirical probability distributions supported in finite samples, and require a cost function between pairs of samples to be provided as an input. Most OT-based dataset distance approaches define this cost function solely in terms of the features of the data, either directly or in a latent embedding space. For example, \citet{muzellec_generalizing_2018} proposed representing objects as elliptical distributions and scaling these computations, while \citet{frogner_learning_2019} extended this to discrete measures. \citet{delon_wasserstein-type_2020} introduced a Wasserstein-type distance for Gaussian mixture models. These approaches are useful mostly in unsupervised learning settings since they do not take into account labels or classes associated with data points. To address this limitation, a different line of work has proposed extensions of OT amenable to supervised or semi-supervised learning settings that explicitly incorporate label information in the cost function.  
%, but this approach did not account for label-to-label similarities.
\citet{courty_domain_2014} used group-norm penalties to guide OT towards class-coherent matches while \citet{alvarez-melis_structured_2018} employed submodular cost functions to integrate label information into the OT objective. For discrete labels, \citet{alvarez-melis_geometric_2020} proposed using a hierarchical OT approach to compute label-to-label distances as distances between the conditional distributions of features defined by the labels.

\paragraph{Predict-then-Optimize (PtO)}
The PtO framework has seen significant advancements in integrating machine learning with downstream optimization. The frameworks proposed by \citet{amos_optnet_2017, donti_task-based_2017, wilder_melding_2019} and \citet{elmachtoub_smart_2022} have been instrumental in this integration. Subsequent work has focused on differentiating through the parameters of optimization problems with various structures, including learning appropriate loss functions \citep{wang_automatically_2020, shah_decision-focused_2022, shah_leaving_2023, bansal_taskmet_2023} and handling nonlinear objectives \citep{qi_integrated_2023, elmachtoub_estimate-then-optimize_2025}. Recent efforts have addressed data-centric challenges within PtO, including including worst-case distribution shifts \citep{ren_decision-focused_2024}, robustness to adversarial label drift \citep{johnson-yu_characterizing_2023} and active learning for data acquisition \citep{liu_active_2025}. While these works propose task-specific learning algorithms, they all share a common underlying principle: dataset similarity. In distribution shifts and label drift, the key challenge lies in the (dis)similarity between training and test datasets, whereas in data acquisition, it concerns the (dis)similarity between the training dataset and the acquisition source.\looseness=-1

%====================================================
\section{Background}

\subsection{Optimal Transport}
OT theory provides an elegant and powerful mathematical framework for measuring the distance between probability distributions by characterizing similarity in terms of correspondence and transfer \citep{villani_optimal_2008, kantorovitch_translocation_1942}. In a nutshell, OT addresses the problem of transferring probability mass from one distribution to another while minimizing a cost function associated with the transportation.\looseness=-1

Formally, given two probability distributions $\alpha$ and $\beta$ defined on measurable spaces $\X$ and $\Y$, respectively, the OT problem seeks a transport plan  $\pi$ (defined as a coupling between $\alpha$ and $\beta$) that minimizes the total transportation cost. According to the Kantorovich formulation \citep{kantorovitch_translocation_1942}, for any coupling $\pi$, the transport cost between $\alpha$ and $\beta$ with respect to $\pi$ is defined as:
\begin{equation}
    d_{T}(\alpha,\beta;\pi) := \int_{\X \times \Y} c(x,y) \dt \pi(x,y),
\end{equation}

\noindent where $c(x,y)$ is the cost function representing the cost of transporting mass from point $x \in \X$ to point $y \in \Y$. The transport cost $d_{T}(\alpha,\beta;\pi)$ defines a distance, known as the \textit{transport distance} with respect to $\pi$, between $\alpha$ and $\beta$. The OT problem then minimizes the transport cost over all possible couplings between $\alpha$ and $\beta$, defining the \textit{optimal transport distance} as:
\begin{equation} 
    d_{OT}(\alpha,\beta;c) := \min\limits_{\pi \in \Pi(\alpha,\beta)}\int_{\X \times \mathcal{Y}} c(x,y) \dt \pi(x,y),
\end{equation}
where $\Pi(\alpha,\beta)$ denotes the set of all possible couplings (transport plans) that have $\alpha$ and $\beta$ as their marginals. This formulation finds the optimal way to transform one distribution into another by minimizing the total transportation cost. \looseness=-1

\subsection{Optimal Transport Dataset Distance}
\label{sec:ot_distance}
In supervised machine learning, datasets can be represented as empirical joint distributions over a feature-label space $\X \times \Y$. OT distances can be used to measure the similarity between these empirical distributions, thus providing a principled way to compare datasets. 
Given two datasets $\D$ and $\D'$ consisting of feature-label tuples $(x,y)$ and $(x',y')$, respectively, the challenge of defining a transport distance between $\D$ and $\D'$ lies in the challenge of defining an appropriate cost function between  $(x,y)$ and $(x',y')$ pairs. A straightforward way to define the feature-label pairwise cost is via the individual metrics in $\X$ and $\Y$ if available. If $d_{\X}$ and $d_{\mathcal{Y}}$ are metrics on $\X$ and $\Y$, respectively, the cost function can be defined as:
\begin{equation}
\label{eq:otdd_cost}
    c_{\X\Y}((x,y),(x',y')) = \big(d_{\X}(x,x')^p + d_{\Y}(y,y')^p\big)^{1/p}
\end{equation}
for $p \geq 1$. This point-wise cost function defines a valid metric on $\X\times \Y$. However, it is uncommon for $d_{\Y}$ to be readily available. To address this, \citet{courty_joint_2017} propose replacing $d_{\Y}(y,y')$ with a loss function $\mathcal{L}(y,y')$ that measures the discrepancy between $y$ and $y'$ while \citet{alvarez-melis_geometric_2020} suggest using a p-Wasserstein distance between the conditional distributions of features defined by $y$ and $y'$  as an alternative to $d_{\Y}(y,y')$. The latter is known as the \textit{Optimal Transport Dataset Distance} (OTDD). We also use this term when referring to the dataset distance $d_{OT}(\D,\D';c_{\X\Y})$.

\subsection{Predict-then-Optimize} \label{sec:back_pto}
The Predict-then-Optimize (PtO) framework involves two sequential steps: prediction and optimization. First, a predictive model $f$ is used to predict costs based on some features $x_1,\ldots,x_N \in \X$, represented as  $\vb*{\hat{y}} = [\hat{y}_1,\ldots,\hat{y}_N] = [f(x_1),\ldots,f(x_N)]$. Second, an optimization model uses these predicted costs $\vb*{\hat{y}}$ as the objective function costs:
\begin{equation}
    M(\vb*{\hat{y}}) := \argmax_w \ g(w;\vb*{\hat{y}}), \hspace{0.4cm} s.t. \hspace{0.2cm} w \in \Omega,
    \label{eq:optmodel}
\end{equation}

\noindent where $\Omega$ is the space of feasible solutions. We assume that $w^{*}_{M}:\R^d \to \Omega$ acts as an oracle for solving this optimization problem, such that $w^{*}_{M}(\vb*{\hat{y}})$ represents the optimal solution for $M(\vb*{\hat{y}})$. However, the solution $w^{*}_{M}(\vb*{\hat{y}})$ is optimal for $M(\vb*{\hat{y}})$ but might not be optimal for $M(\vb*{y})$, where $\vb*{y}$ represents the true costs. \looseness=-1

Given a hypothesis function $f: \X \to \Y$, we measure its performance on the optimization problem $M(\vb*{y})$ using the predicted cost vector $\vb*{\hat{y}} = [f(x_1), \ldots, f(x_N)]$ and the true cost vector $\vb*{y} = [y_1, \ldots, y_N]$. This is quantified as the \textit{decision quality} $q(\vb{\hat{y}}, \vb*{y}) := g(w^{*}_{M}(\vb{\hat{y}}); \vb*{y})$, reflecting the quality of decisions made using $w^{*}_{M}(\vb{\hat{y}})$ as a solution to $M(\vb*{y})$. The \textit{decision quality regret}, which evaluates the performance of $f$, is defined as:
\begin{equation}
\label{eq:dqreg}
    q_{\text{reg}}(\vb*{\hat{y}}, \vb*{y}) = |q(\vb*{y}, \vb*{y}) - q(\vb*{\hat{y}}, \vb*{y})|.
\end{equation}

The goal of decision-focused learning\cite{wilder_melding_2019} in a PtO task is to learn a predictive model $f_\theta$ that minimizes the decision quality regret, ensuring that the decisions derived from the predictions are as close to optimal as possible.

%====================================================
\section{Motivating example}

To illustrate the role of decisions in PtO task comparisons, we look at correspondence between task similarity and zero-shot transfer performance in a simple PtO task: the Linear Model Top-K setting from \citet{shah_decision-focused_2022}. This task consists of two stages: \looseness=-1

\begin{itemize}[leftmargin=0.3cm]
    \item[] \textit{Predict:} Given a resource’s feature $x_n \sim \Pd_\X$, where $\Pd_\mathcal{X} = \text{Unif}[-1, 1]$, a linear model predicts its utility $\hat{y}_n$, where the true utility follows $y_n = p(x_n)$, a cubic polynomial. Predictions for $N$ resources form $\hat{\bm{y}} = [\hat{y}_1, \ldots, \hat{y}_N]$.
    \item[] \textit{Optimize:} Select the top $K=1$ resource by solving $M(\hat{\bm{y}}) = \max_{\bm{z} \in [0,1]^N} \{\bm{z} \cdot \sigma_x(\hat{\bm{y}})\}$ such that $||\bm{z}||_0 = K$, where $\sigma_x$ orders $\hat{\bm{y}}$ in ascending order of $\bm{x} = [x_1, \ldots, x_N]$.
\end{itemize}

\noindent We analyze this task under target shifts—where label distributions change while feature distributions remain constant— parametrized by $\gamma$, where the shifted utility function is given by $p_\gamma(x) = 10(x^3 - \gamma x)$. We define two source domains, $A$ and $B$, with shifts $\gamma=0$ and $\gamma=1.2$, respectively. The target domain $C$ is characterized by $\gamma=0.65$.\footnote{We choose $\gamma=0.65$ for consistency with \citet{shah_decision-focused_2022}.}

\begin{figure}[h!]
    \centering
    \includegraphics[width=\columnwidth]{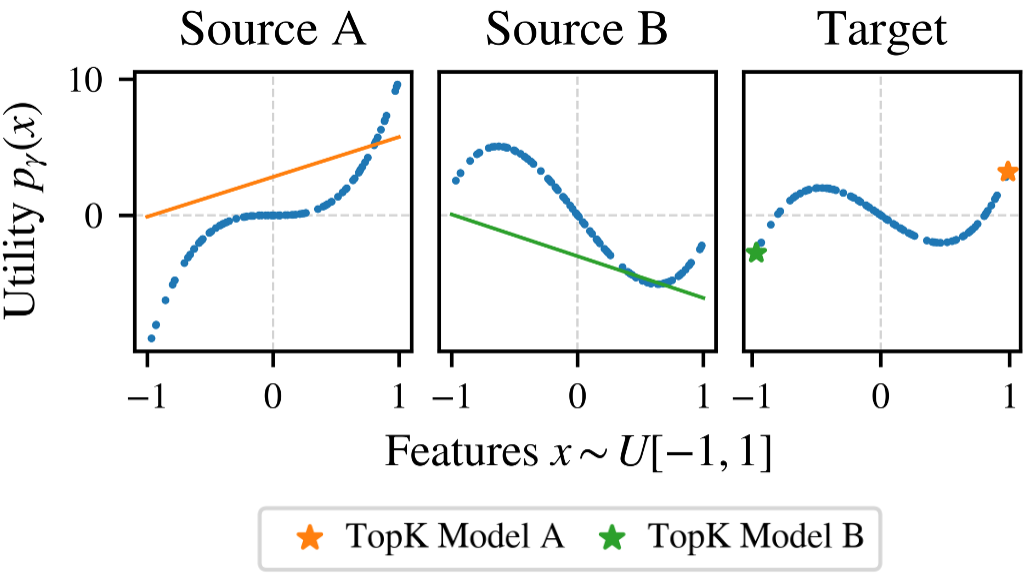}
    \caption{Linear Model Top-K instances under target shift}
    \label{fig:example}
\end{figure}

Assume we have only a few instance $\mathcal{D}_C = \{(\bm{x}_i,\bm{y}_i)\}_{i=1}^{n} \sim C$, but want to learn how to perform Top-K selection in domain $C$ when the true costs $\bm{y}$ are unknown. This limited data is insufficient for directly learning weights in this domain. However, we have access to datasets $\mathcal{D}_A$ and $\mathcal{D}_B$ drawn form source domains $A$ and $B$, respectively, allowing us to learn decision-focused weights $\theta_A$ and $\theta_B$.\footnote{We use \url{github.com/sanketkshah/LODLs}} The key question is: which weights, $\theta_A$ or $\theta_B$, should be used for the PtO task in $\mathcal{D}_C$?

A natural approach is to use the OTDD to identify the source dataset closest to the target. However, in this case, the computed distances $d_{OT}(\mathcal{D}_A, \mathcal{D}_C;c_{\X\Y})$ and $d_{OT}(\mathcal{D}_B, \mathcal{D}_C;c_{\X\Y})$ are equal, suggesting no preference between $\theta_A$ and $\theta_B$. Yet, in practice, $\theta_A$ yields zero regret on $\mathcal{D}_C$, while $\theta_B$ results in a regret close to 4, making $\theta_A$ the clear choice for the PtO task on $\mathcal{D}_C$ (see Appendix Fig.~\ref{app-fig:regression} for details). Figure \ref{fig:example} illustrates this discrepancy: the model with $\theta_A$ (orange) successfully selects the correct Top-K resource in $\mathcal{D}_C$, while the model with $\theta_B$ (green) fails to do so. Since regret differs significantly, dataset distances should reflect that $\mathcal{D}_C$ is closer to $\mathcal{D}_A$ than $\mathcal{D}_B$. We argue that feature-label distances alone are insufficient, and incorporating decision components is necessary for distances to accurately reflect similarities, and hence adaptability, in PtO tasks.\looseness=-1

%====================================================
\section{Decision-aware Dataset Distance}

A dataset for a PtO task with downstream optimization $M(\cdot)$ and oracle $w^{*}_{M}$ consists of feature-label-decision triplets $(x, y, z) \in \X \times \Y \times \Omega$, where the decision $z$ is the precomputed optimal solution to the optimization task parametrized by the true label $y$, i.e., $z=M(y)$. Our objective is to formalize a notion of similarity between PtO tasks by defining a distance $d(\D,\D')$ between datasets $\D=\{(x,y,M(y))\}_{(x,y)\sim \Pd}$ and $\D'=\{(x,y,M'(y))\}_{(x,y)\sim \Pd'}$ for any distributions $\Pd$ and $\Pd'$ over the joint feature-label space $\X \times \Y$ and optimization problems $M$ and $M'$ parametrized by $\Y$. 

OT provides a natural framework for comparing datasets by leveraging the geometry of the underlying space and establishing correspondences between distributions. It has been used as the foundation for OTDD, a dataset distance defined over features and labels \citep{alvarez-melis_geometric_2020}. We extend this idea to PtO tasks, where dataset distances must also account for differences in decisions arising from the downstream optimization process. Unlike standard settings where similarity is assessed based only on feature-label distributions, PtO tasks introduce an additional layer of complexity: decisions $z$ are solutions to an optimization problem dependent on $y$, and their quality directly impacts task performance.

In the following sections, we formalize our proposed decision-aware dataset distance by extending OTDD to incorporate decision quality. This formulation provides a principled way to compare PtO datasets, ensuring that the resulting distance reflects meaningful differences in feature-label-decision distributions while remaining sensitive to the structure of the underlying optimization problem.\looseness=-1

%====================================================

\subsection{Dataset Distance Formulation}
\label{sec:decision_distance}
To apply OT to datasets in PtO settings, we need a well-defined metric for the joint space of features, labels and decisions, $\Z := \X \times \Y \times \Omega$, to serve as the ground cost function in the OT problem. A natural approach is to construct distances in $\Z$ by combining metrics from the feature space $\X$, the label space $\Y$, and the decision space $\Omega$. In most well-studied PtO settings, $\X$ and $\Y$ are equipped with metrics $d_\X$ and $d_\Y$, respectively, providing a natural foundation for measuring feature-label distances.
% When $\Y$ lacks a natural metric, one can use $d_\X$ to take distances between labels based on their association with features as done in OTDD. 
However, defining an appropriate metric for the decision space $\Omega$ requires special consideration. 

While some decision spaces naturally admit standard metrics, others—such as those arising in resource allocation or scheduling—do not align with conventional distance measures. Even when a metric exists for $\Omega$, it may fail to capture decision quality regret, the ultimate objective in PtO tasks. For example, in a $p \times q$ grid, Euclidean or Manhattan distances can measure geometric differences between paths but fail to account for task-specific objectives, such as minimizing costs or maximizing safety.

To ensure the ground cost function for $\Z$ properly reflects both decision quality and feature-label relationships, we introduce the concept of \textit{decision quality disparity}. This extends traditional metrics by comparing decisions not just in terms of their spatial or structural differences but also in terms of their effectiveness under different labels. Specifically, decision quality disparity measures the extent to which two decisions $z, z' \in \Omega$ differ in performance when evaluated under labels $y$ and $y'$ respectively.

\begin{definition} For an optimization problem $M(\cdot)$ with objective function $g$, the \textit{decision quality disparity} function $\ldq(\ \cdot \ ; y, y'): \Omega^2 \to \R$ measures the difference in decision quality between two decisions $z,z'\in \Omega$ given the labels $y,y' \in \Y$. It is defined as:
\begin{align}
    &\ldq(z,z';y,y') := \abs{g(z;y) - g(z';y')}.
\end{align}
\label{def:dqd}
\end{definition}

Note that decision quality regret (Section~\ref{sec:back_pto}) is a special case of decision quality disparity, where $\dqreg(\hat{y}, y)=\ldq(w^*(\hat{y}), w^*(y); y, y)$ for an optimization oracle $w^*$. We use decision quality disparity to define a point-wise distance in the joint feature-label-decision space $\Z$. The resulting ground cost function $c_{PtO}^{\bm{\alpha}}$ for the OT problem is given by:
\begin{align}
\label{eq:cpto}
    c_{PtO}^{\bm{\alpha}}((x,y,z),(x',y',z')) := & \ \ \alpha_X \cdot d_{\X}(x,x') \notag \\&+ \alpha_Y \cdot d_{\Y}(y,y') \\&+ \alpha_W \cdot \ldq(z,z';y',y'), \notag
\end{align}
\noindent for $\bm{\alpha} = [\alpha_X,\alpha_Y,\alpha_W] \in \R^{3}_{\geq 0}$ such that $||\bm{\alpha}||=1$. Here, $d_{\X}$ and $d_{\Y}$ represent metrics for the feature space $\X$ and label space $\Y$, while $\ldq$ captures differences in decisions. The additive combination in $c_{PtO}$ ensures simplicity and validity as a metric, with each component independently reflecting a distinct aspect of similarity. This design avoids complex, application-specific interactions and prioritizes interpretability. In the appendix, we show that $c_{PtO}$ is a proper distance in $\Z$. Notably, we set $\alpha_Y > 0$ to analyze scenarios where both labels and decisions contribute to PtO similarity versus cases where labels may be redundant.

We extend this point-wise distance to a distance between datasets $\D$ and $\D'$ by solving the OT cost with ground cost $c_{PtO}^{\bm{\alpha}}$, denoted as $d_{OT}(\D,\D';c_{PtO}^{\bm{\alpha}})$. We refer to this distance as the \textit{Optimal Transport Decision-Aware Dataset Distance ($OTD^3$)}.

\begin{proposition}
\label{prop:metric}
For any $\bm{\alpha}=(\alpha_X,\alpha_Y, \alpha_W)$ with $\alpha_X,\alpha_Y, \alpha_W >0$, $d_{OT}(\D, \D'; c^{\bm{\alpha}}_{PtO})$ is a valid metric on $\mathcal{P}(\X \times \Y \times \Omega)$, the space of measures over joint distributions of features $\X$, labels $\Y$, and decisions $\Omega$. If $\alpha_Y=0$, $d_{OT}(\D, \D'; c^{\bm{\alpha}}_{PtO})$ is at least a pseudometric.
\end{proposition}

This decision-aware dataset distance compares decisions $z$ and $z'$ by evaluating their decision quality disparity in $\mathbb{R}$ relative to a pair of fixed labels, rather than directly comparing them in the decision space $\Omega$. Intuitively, comparing decisions based on their quality, i.e., comparing $g(z; y)$ with $g(z'; y)$, rather than comparing $z$ and $z'$ directly using some metric in $\Omega$, if available, is reasonable because similar decisions might yield significantly different outcomes in the objective function. In Section~\ref{sec:bound} we show that comparing decision in this way offers a principled means of assessing adaptation success of PtO tasks across distributions in the feature-label-decision space.

\paragraph{Component Weights are Task-Specific Hyperparameters.} The weights $\bm{\alpha}$ on the ground cost component (Eq.~\ref{eq:cpto}), are pivotal in defining the OTD$^3$, offering a flexible framework to account for the varying importance of features, labels, and decisions in PtO tasks. Unlike previous OT-based dataset distances that did not differentiate between the weights of feature and label components in the ground cost function—often because both were measured in the same space \cite{alvarez-melis_geometric_2020} or were weighted equally \cite{courty_joint_2017}— our method allows for distinct weights, enabling a more nuanced evaluation of dataset similarity tailored to each specific task. This flexibility ensures that the distance metric reflects the relative significance of each dataset component according to its impact on the PtO task, which can vary widely in practice depending on the application.

The impact of each component—features, labels, and decisions—on the overall distance can vary across PtO tasks. In particular, the decision and label components may sometimes capture overlapping structure. When such alignment occurs, the added value of decision information may be diminished, while in other cases, decisions encode complementary information. This mirrors the intuition from multivariate modeling where high correlation between variables can reduce the sensitivity to their individual weights. While our current formulation provides flexibility via weighting, understanding when and how much each component contributes remains an open and important question. We return to this empirically in Section~\ref{sec:pred_transfer}.

\subsection{Decision Regret Adaptation Bound} \label{sec:bound}
Given source and target distributions $\Pd_S$ and $\Pd_T$ over $\X \times \Y$, we study domain adaptation from $\Pd_S$ to $\Pd_T$ in a PtO framework where decisions are generated by a downstream optimization problem $M(\cdot)$ parametrized in $\Y$. Let $f:\X \to \C$ be a labeling function. We define the expected cost of $f$ under a distribution $\Pd$ over $\X \times \C$ with respect to any cost function $l:\C \times \C \to \R$ as 
\begin{equation}
 err(f; l, \Pd) := \E_{(x,y)\sim \Pd} \ l(f(x),y).
\end{equation}
In PtO tasks, the performance of $f$ over a distribution $\Pd$ is quantified as the \textit{expected decision quality regret}, given by $err(f; \dqreg, \Pd)$. Our goal is to bound this error on the target distribution, $\text{err}(f; \dqreg, \Pd_T)$, in terms of the distance between $\Pd_T$ and the source distribution $\Pd_S$. We use $OTD^3$ to achieve this.

Prior work by \cite{courty_joint_2017} provided adaptation bounds for an expected target error $\text{err}(f; l, \Pd_T)$ with a cost function $l$ that is bounded, symmetric, $k$-Lipschitz, and satisfies the triangle inequality. However, decision quality regret $\dqreg$, the key cost function in PtO tasks, is inherently non-symmetric, making these bounds inapplicable to $\text{err}(f; \dqreg, \Pd)$. To address this, we introduce the notion of  \textit{decision quality disparity} $l_q$ (Definition~\ref{def:dqd}) to bound decision quality regret $\dqreg$. Additionally, we assume that the decision quality function $q$ has a bounded rate of change with respect to both the predicted and true cost vectors (Assumption \ref{assump:1}). Under these conditions, we derive an adaptation bound for $err(f; \dqreg, \Pd_T)$ using the OTD$^3$ (Theorem \ref{theo1}). As demonstrated in lemmas \ref{lemma1} and \ref{lemma2} in the Appendix, Assumption \ref{assump:1} holds for common PtO task structures.

\begin{assumption}
\label{assump:1}
    The decision quality function $q$ is $k_1,k_2$-Lipschitz. This means that for any $y,y^{*},z,z^{*} \in \C$ the following inequality holds:
    $$|\dq(y,y^{*}) - \dq(z,z^{*})| \leq k_1\norm{y-z} + k_2\norm{y^{*}-z^{*}}$$.
\end{assumption}

\begin{definition}[\citet{courty_joint_2017}]
 Let $\mu_1$ and $\mu_2$ be distributions over some metric space $\X$ with metric $d_{\X}$. Let $\Pi(\mu_1,\mu_2)$ be a joint distribution over $\mu_1 \times \mu_2$. Let $\phi: \R \to [0,1]$. A labeling function $f: \X \to \R$ is $\phi$-\textit{Lipschitz transferable with respect to $\Pi$} if for all $\lambda >0$: $$\Pr_{(x_1,x_2)\sim \Pi(\mu_1,\mu_2)}[|f(x_1)-f(x_2)| > \lambda d_{\X}(x_1,x_2)] \leq \phi(\lambda).$$
\end{definition}

\begin{theorem}
\label{theo1} Suppose Assumption \ref{assump:1} holds for an optimization problem $M(\cdot)$ with optimization oracle $w^*$. Let $f:\X \to \C$ be a labeling function, and define the distributions $\Pd_{T}^{f} := (x,y,w^*(f(x)))_{(x,y) \sim \Pd_T}$ and $\Pd_{S}^{*} := (x,y,w^*(y))_{(x,y) \sim \Pd_S}$ over the joint feature-label-decision space $\Z$. Let $\Pi^*$ denote the optimal coupling for the OT problem with ground cost $c_{PtO}^{\bm{\alpha}}$ between $\Pd_{T}^{f}$ and $\Pd_{S}^{*}$. Let $\Tilde{f}$ be a labeling function that is $\phi$-Lipschitz transferable with respect to $\Pi^*$. 
Assume that the feature space $\X$ is bounded by $K$ and that $\Tilde{f}$ is $l$-Lipschitz, satisfying $|\Tilde{f}(x_1)-\Tilde{f}(x_2)| \leq 2lK = L$. 

For any $\lambda > 0$ and $\alpha_W \in (0,1)$ such that $(\lambda k_1 + k_2 + 1)\alpha_W = 1$, with $\alpha_X = \lambda k_1 \alpha_W$ and $\alpha_Y = k_2\alpha_W$, the following bound holds with probability at least $1 - \delta$:
\begin{align*}
    err(f; \dqreg, \Pd_T) \leq &\ err(\Tilde{f}; \dqreg, \Pd_S) + err(\Tilde{f}; \dqreg, \Pd_T)  \\
    &+ k_1L\phi(\lambda) + \frac{1}{\alpha_W}d_{OT}(\Pd_{T}^{f}, \Pd_{S}^{*} ; c_{PtO}^{\bm{\alpha}}).
\end{align*}
\end{theorem}

The proof of Theorem~\ref{theo1} is provided in the supplementary material. The first two terms in the bound represent the joint decision quality regret minimizer across the source and target distributions. This indicates that successful domain adaptation in the PtO framework requires predictions that achieve low regret in both domains simultaneously. This result aligns with the findings of \cite{courty_joint_2017,mansour_domain_2009} and \cite{ben-david_theory_2010} in the context of domain adaptation for supervised learning. The third term $k_1L\phi(\lambda)$ captures the extend to which Lipschitz continuity between the source and target distributions may fail.

The final term measures the discrepancy between the source domain $\Pd^{*}_{S}$ and the predicted target domain $\Pd^{f}_{T}$ using the optimal transport distance between their joint distributions of features, labels, and decisions. The bound relies on two key parameters: $\lambda$ and $\alpha_W$. $\lambda$ controls the weight of the Lipschitz term and is valid for any $\lambda > 0$, while $\alpha_W$ determines the weight assigned to decisions in the convex combination $c_{PtO}^{\bm{\alpha}}$. Note that the bound holds for any combination of weights $\alpha_X, \alpha_Y, \alpha_W$, as $\lambda$ can always be adjusted to ensure a valid convex combination. 

Our approach recognizes the necessity of incorporating decisions into dataset distances when used for domain adaptation purposes for PtO tasks. Our OT-based dataset distance, defined by the ground cost function $c_{PtO}^{\bm{\alpha}}$, jointly accounts for differences in all key components—features, labels, and decisions—providing a comprehensive measure that is meaningful for adaptability of PtO tasks.

%====================================================

\section{Experimental settings} \label{sec:exp_settings}

We conduct experiments on three PtO settings with diverse structures and sensitivity to distribution shifts, making them well-suited for analyzing dataset distances in domain adaptation. Additional details are provided in the Appendix.

\paragraph{Linear Model Top-K \citep{shah_decision-focused_2022}.} This setting involves training a linear model to map features \( x_n \sim U[-1,1] \) to true utilities based on a cubic polynomial \( p(x_n) = 10(x_{n}^3 - 0.65x_n) \). The downstream task is selecting the \( K \) elements with highest utility. We introduce synthetic distribution shifts by modifying the original feature-label distribution \( \Pd = (\text{Id}, p)_*U[-1,1] \). Specifically, for various values of \( \gamma \in [0,1.3] \), we define the feature-label distributions \( \Pd_\gamma = (\text{Id}, p_\gamma)_*U[-1,1] \) where \( p_{\gamma}(x_n) = 10(x_{n}^3 - \gamma x_n) \), using \( \Pd_{0.65} \) as the target distribution.

\paragraph{Warcraft Shortest Path \citep{vlastelica_differentiation_2020}.} This task involves finding the minimum-cost path on RGB grid maps from the Warcraft II tileset dataset, where each pixel has an unknown travel cost. The goal is to predict these costs and then determine the optimal path from the top-left to the bottom-right pixel. The target distribution $\mathcal{P}$ is defined over $\mathbb{R}^{d \times d \times 3} \times \mathbb{R}^{p \times p}$, with $d = 96$ and $p = 12$. To simulate distribution shifts, we generate synthetic distributions $\Pd_\gamma$ by uniformly sampling pixel class costs from the same range as $\mathcal{P}$. \looseness=-1

\paragraph{Inventory Stock Problem \citep{donti_task-based_2017}.} This task involves determining the order quantity \(z\) to minimize costs given a stochastic demand \(y\), influenced by features \(x\). The cost function $f_{\rm stock}$ includes linear and quadratic costs for both ordering and deviations (over-orders  and under-orders) from demand.
% Given a probability model \(p(y|x; \theta)\), the optimization problem is: $\min_z \mathbb{E}_{y \sim p(y|x; \theta)} \left[ f_{\text{stock}}(y, z) \right]$.For discrete demands \(d_{1}, \ldots, d_{k}\)  with probabilities \(p\left(y = d_{i}| x; \theta\right)\) the optimization problem can be formulated as a joint quadratic program. 
We generate problem instances by randomly sampling \(x \in \mathbb{R}^n\) and then generating \(p(y|x; \theta)\) according to \(p(y|x; \theta) \propto \exp((\theta^T x)^2)\). Distribution shifts are introduced in features \(x\) and labels \(y\): \(x\) is sampled from a Gaussian distribution with a mean sampled from \(U[-0.5, 0.5]\), and \(\theta\) is also sampled from a Gaussian distribution.\looseness=-1

\begin{figure*}
    \centering
    \begin{subfigure}[b]{0.3\textwidth}
        \centering
        \includegraphics[width=0.9\textwidth]{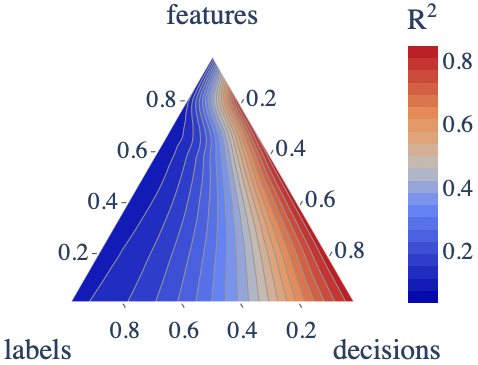} 
        \caption{Linear Model Top-K}
        \label{fig:triplots_topk}
    \end{subfigure}
    \hfill
    \begin{subfigure}[b]{0.3\textwidth}
        \centering
        \includegraphics[width=0.9\textwidth]{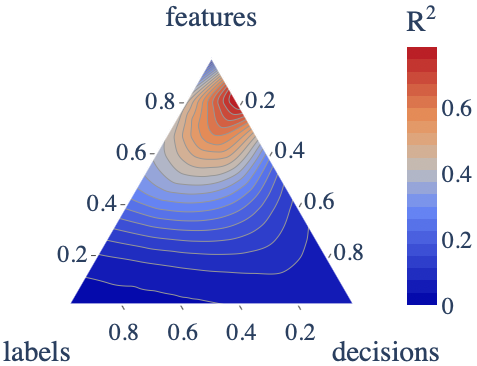}
        \caption{Warcraft Shortest Path}
        \label{fig:triplots_warcraft}
    \end{subfigure}
    \hfill
    \begin{subfigure}[b]{0.3\textwidth}
        \centering
        \includegraphics[width=0.9\textwidth]{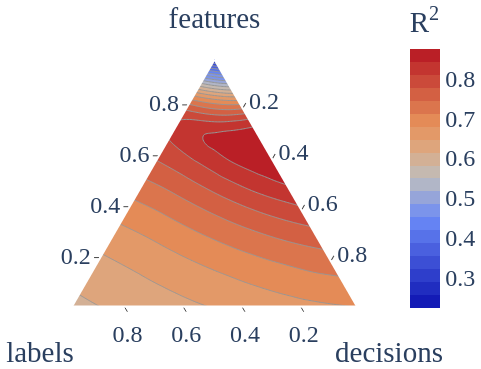}
        \caption{Inventory Stock}
        \label{fig:triplots_inventory}
    \end{subfigure}
    \caption{\textbf{Component weighting and transferability prediction}. The color scale represents the R-square value from a linear regression of OTD$^3$—with all possible weight combinations for features, labels, and decisions—against regret transferability. The left border of the triplot shows R-square values when using OTD$^3$ with $\alpha_W=0$, equivalent to OTDD.}
    \label{fig:triplots}
\end{figure*}

\section{Experiments} 
\label{sec:Experiments}

\subsection{Selecting Source Datasets for Transfer Learning} \label{sec:pred_transfer}

Dataset distances for feature-label datasets, such as OTDD, have shown to be predictive of classification error/accuracy transfer—i.e., the error/accuracy on a target dataset $\D_{T}^{test}$ for a model adapted from a source $\D_{S}$ to a target $\D_{T}$. \citet{alvarez-melis_geometric_2020} demonstrated that OTDD effectively predicts transferability and used this measure to select the best source dataset in a transfer learning task. We extend this source dataset selection experiment to the PtO setting, evaluating how well the OTD$^3$ predicts \textit{regret transferability} between PtO tasks.\looseness=-1

We analyze the correlation between the distance from a source dataset $\D_S$ to a target dataset $\D_{T}$ and the regret incurred on unseen target data $\D_{T}^{test}$ when adapting a model from $\D_S$ to $\D_{T}$—\textit{i.e.} pretraining on $\D_S$ and fine-tuning on $\D_{T}$. We compare OTD$^3(\D_S,\D_{T})$ against regret transferability $\mathcal{T}$, which quantifies the relative reduction in regret when transferring from $\D_S$ to $\D_{T}$:
\begin{equation*}
\mathcal{T}(S \to T) = 100 \times \frac{\text{reg}(\D_{T}) - \text{reg}(\D_{S} \to \D_{T})}{\text{reg}(\D_{T})},
\end{equation*}
\noindent where $\text{reg}(\D_{T})$ represents the mean regret when training directly on $\D_{T}$, and $\text{reg}(\D_{S} \to \D_{T})$ represents the mean regret when adapting from $\D_S$ to $\D_{T}$. Each regret term is computed on $\D_{T}^{test}$, ensuring that transferability is evaluated based on the model's performance on unseen target data. \looseness=-1

For every experimental setting we generate $K$ source datasets $\D_{S_1}, \ldots, \D_{S_K}$, each sampled from a different distribution $\Pd_{S_i}$, along with training and test datasets $\D_{T}$ and $\D_{T}^{test}$ drawn from a target distribution $\Pd_{T}$. For each source-target pair $(\D_{S_i}, \D_T)$, we compute the regret transferability $\mathcal{T}(S_i \to T)$ by training models using standard DFL approaches (Appendix \ref{app:experiments}) and analyze its relationship with the OTD$^3$. \looseness=-1

\paragraph{Predicting transferability.} Figure~\ref{fig:triplots} shows the correlation strength ($R^2$ from linear regression) between regret transfer and dataset distance for different weighting combinations $\bm{\alpha}$. In the Linear Model Top-K and Warcraft settings, incorporating the decision component ($\alpha_W > 0$) significantly enhances the correlation between dataset distance and regret transfer, even when the label component is excluded ($\alpha_Y = 0$). Conversely, omitting the decision component ($\alpha_W = 0$, left side of the triplot) weakens this correlation. This trend is further emphasized when comparing the highest achievable correlation. In Warcraft, the OTD$^3$ with maximizing weights is far more predictive of regret transfer than the OTDD (or the OTD$^3$ with $\alpha_W=0$) under its best weighting (Figure~\ref{fig:regression}). The best-performing weights in this case were $\alpha_X=0.8$ and $\alpha_Y=0.2$ for the OTDD, and $\alpha_X=0.75$, $\alpha_Y=0$, and $\alpha_W=0.25$. We denote these optimized versions as the OTDD$*$ and the OTD3$*$.

\begin{figure}
    \centering
    \includegraphics[trim=0 10 0 0,clip, width=\columnwidth]{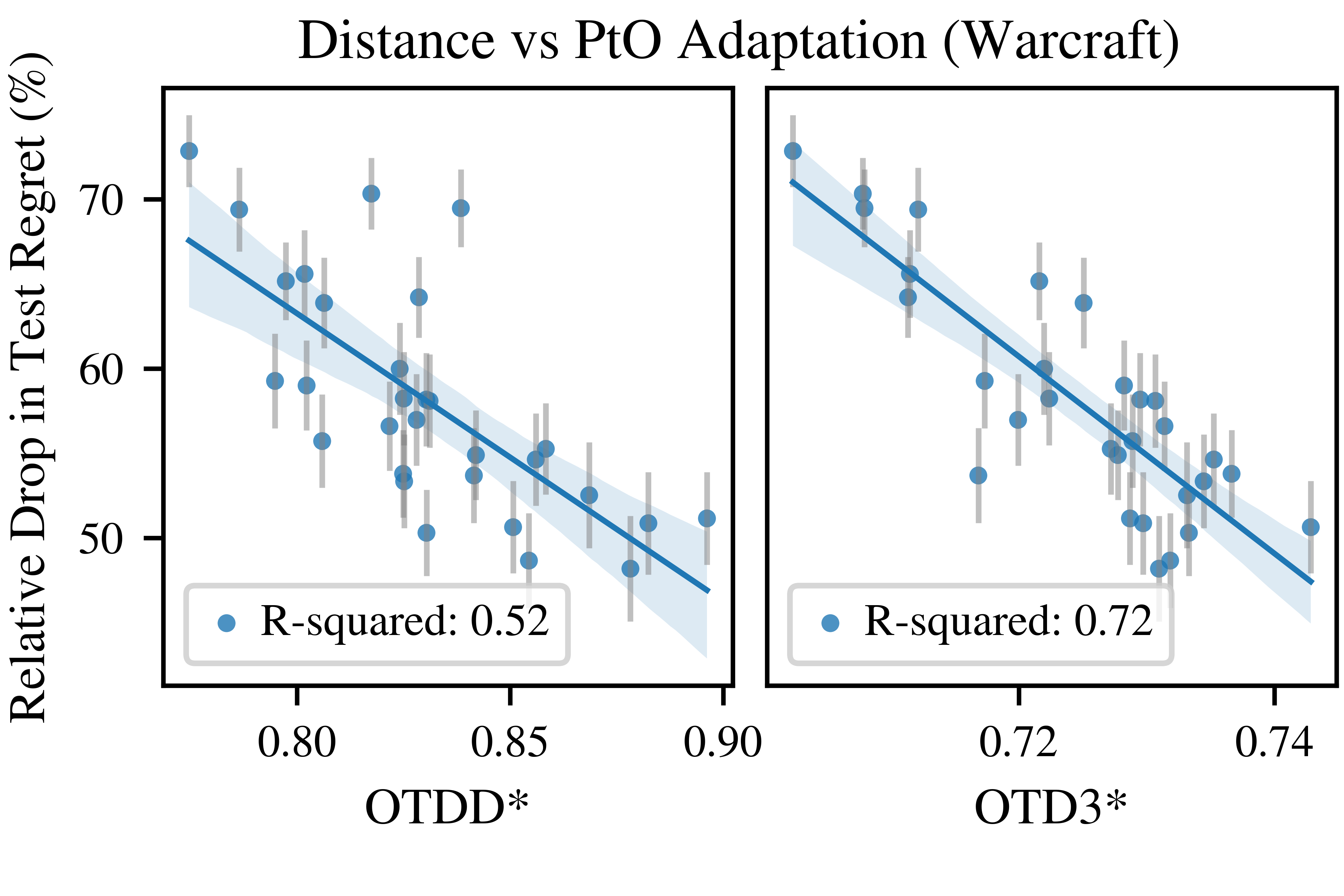}
    \caption{Dataset distance vs PtO adaptation in Warcraft. Results for a target dataset with 100 samples against 30 source datasets with 1,000 samples. Dataset distances OTDD and OTD$^3$ are computed with weights that maximize the correlation between distance and regret transferability.}
    \label{fig:regression}
\end{figure}

In Figure~\ref{fig:correlation_ntarget} we extend our analysis to varying sizes of the target dataset, ranging from 10 to 100 samples, which are used for dataset distance computation and fine-tuning, while keeping the source datasets and target test set fixed at 1,000 samples each. We compute dataset distances using a three-dimensional weight grid and compare the correlation achieved with equal input-output weighting ($\alpha_X=0.5, \alpha_Y=0.5$ for OTDD, $\alpha_X=0.5, \alpha_Y=0.25, \alpha_W=0.25$ for OTD$^3$) against OTDD* and OTD$^3$*.\looseness=-1

Incorporating decisions into the dataset distance, with appropriate weighting, consistently improves the predictability of PtO transferability across all target sample sizes. While tuning the decision component weight significantly boosts correlation, rapidly reaching $R^2>0.6$, OTD$^3$ outperforms OTDD from as few as 30 target samples onward, demonstrating its effectiveness even without extensive data for weight optimization.

\begin{figure}
    \centering
    \includegraphics[trim=0 7 0 5,clip, width=\columnwidth]{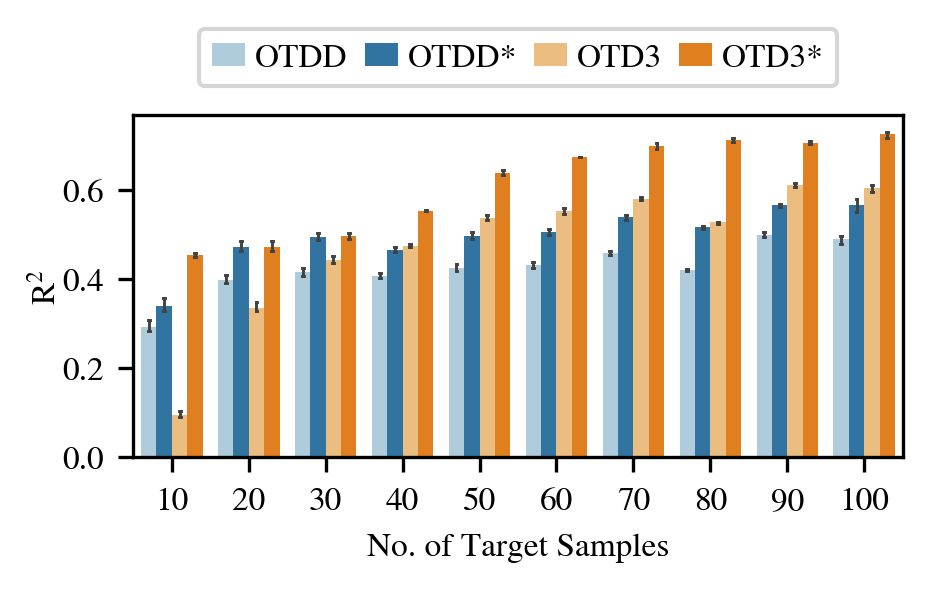}
\caption{Correlation between dataset distance and regret transferability ($R^2$) vs. target sample size, for four distance variants: OTDD with equal feature-label weights and optimized weights (OTDD*), and OTD$^3$ with equal output weights (0.5, 0.25, 0.25) and optimized weights (OTD$^3$*).}
    \label{fig:correlation_ntarget}
\end{figure}

\paragraph{Decision vs label component.} The advantage of including the decision component over the label component is less pronounced in the Inventory Stock problem (Fig.~\ref{fig:triplots_inventory}). Here, either the label or decision component with features still maintains a strong correlation between regret transfer and dataset distance. To explore this further, we examine how differences in the label space $d_y(y,y')$ correlate with differences in the decision space $l_q(y,y',z,z')$. In the Inventory Stock problem, there is a strong correlation between these differences (Appendix Fig. \ref{fig:correlation_inventory}), suggesting that decisions are closely tied to the labels. In contrast, the Warcraft domain lacks this strong correlation (Appendix Fig. \ref{fig:correlation_warcraft}), making the decision component more critical for accurately predicting transferability.

\subsection{Characterizing Target Shift Impact}
Target shift—where label distributions change while feature distributions remain constant—creates mismatches between training and test datasets, often degrading performance in supervised learning. However, our experimental results (Fig.~\ref{fig:regression}[left]) show that some source datasets with significant target shift—characterized by high feature-label distance—can still achieve low regret in the PtO task. This suggests that target shift may not impact PtO performance in the same way it affects purely predictive tasks. \looseness=-1

To further explore the impact of target shift in PtO tasks, we analyze the Warcraft setting under two downstream optimization tasks: minimizing path cost alone and minimizing both path cost and length. We apply the same transfer learning experiment from Section~\ref{sec:pred_transfer} it to these two tasks. Although the same target shifts are applied on both tasks, their effect on PtO transferability is less severe for minimizing path cost and length compared to minimizing cost alone (Fig.~\ref{fig:jointplot}). Our decision-aware dataset distance, using weights from Section~\ref{sec:pred_transfer}, effectively captures this behavior. The distance distribution for the task less impacted by the target shift is more left-skewed (Fig.~\ref{fig:joint_plot_decision}). In contrast, the dataset distance that only accounts for features and labels, is unable to differentiate between these two tasks (Fig.~\ref{fig:joint_plot_label}). \looseness=-1
\begin{figure}[]
    \centering
    \begin{subfigure}[b]{0.53\columnwidth}
        \centering
        \includegraphics[height=1.75in]{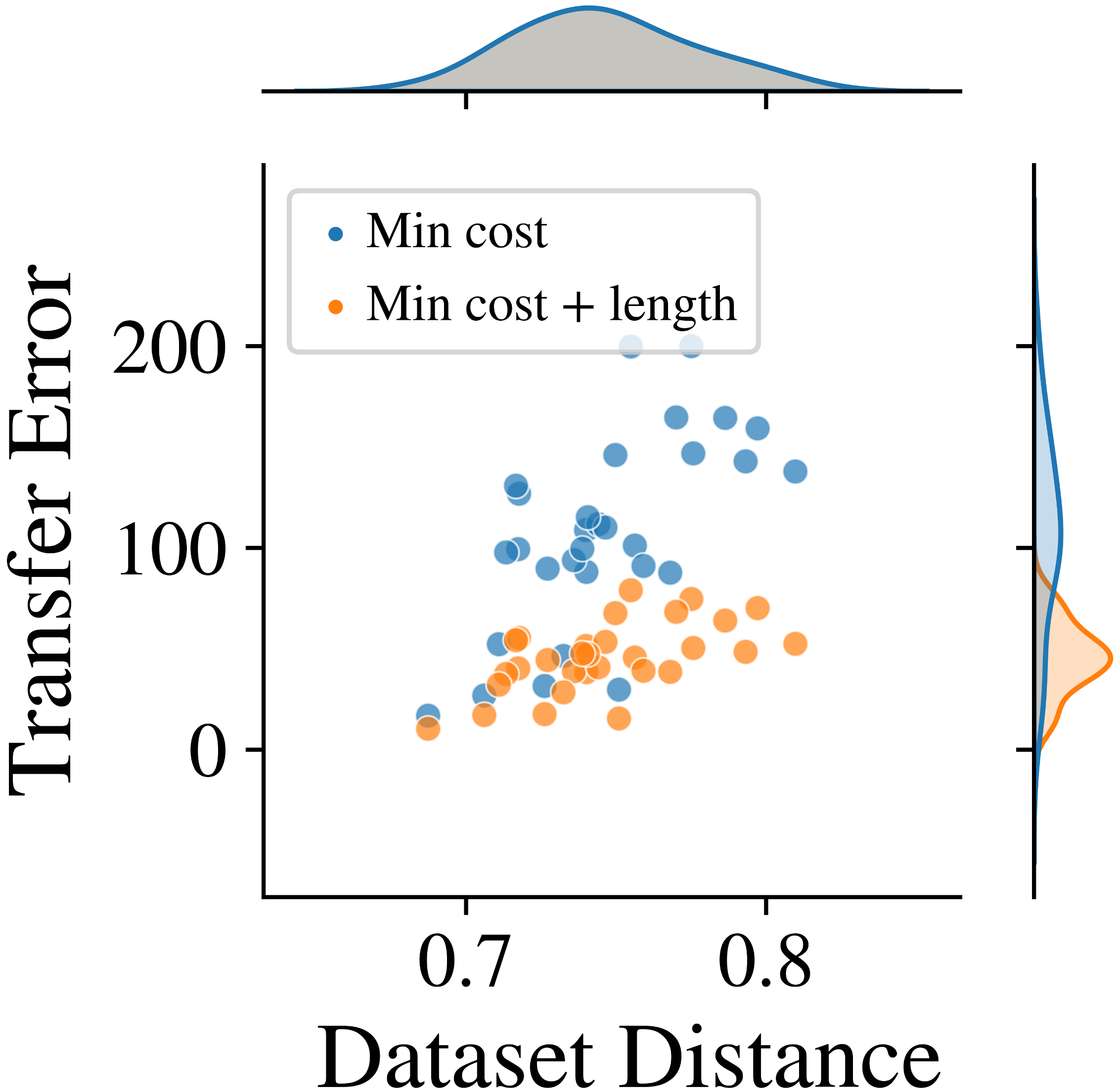} 
        \caption{Not including decisions}
        \label{fig:joint_plot_label}
    \end{subfigure}
    \hfill
    \begin{subfigure}[b]{0.44\columnwidth}
        \centering
        \includegraphics[height=1.75in]{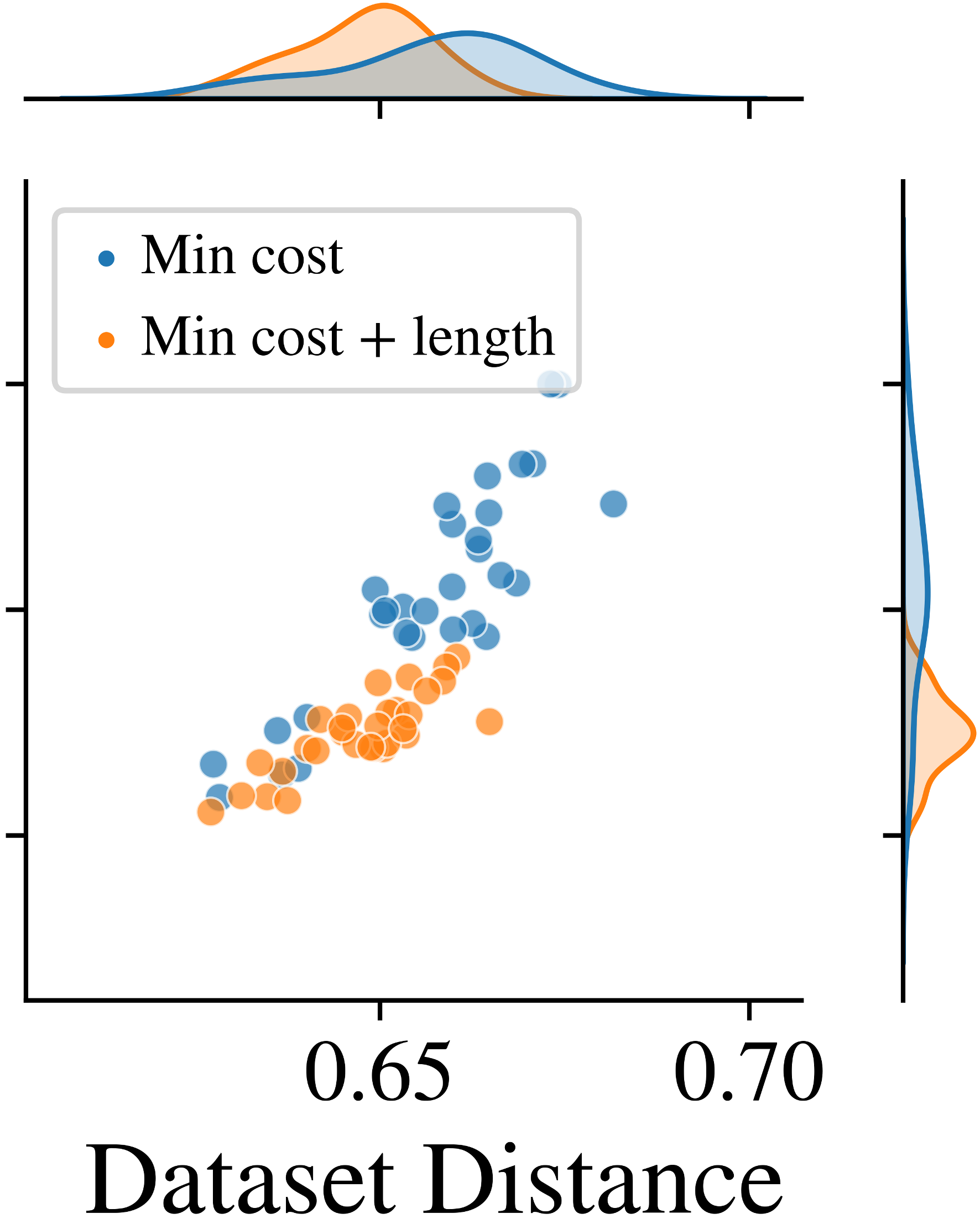}
        \caption{Including decisions}
    \label{fig:joint_plot_decision}
    \end{subfigure}
    \caption{Distance vs. Adaptation for two tasks in the Warcraft setting. Dataset distance is computed (a) without incorporating decisions, and (b) with decision incorporation.}
    \label{fig:jointplot}
\end{figure}

% The effectiveness of domain adaptation under target shift in PtO tasks depends on how target shift influences downstream decision. If decisions remain mostly invariant despite changes in target distribution, the tasks can exhibit better transferability, demonstrating that the impact of target shift on PtO tasks can be less detrimental than in other supervised learning scenarios. This effect is captured by our decision-aware dataset distance, where distributions with target shift but almost invariant decisions appear closer or more similar, reflecting better transferability between domains.

% \subsection{Selecting Data for Labeling}
% We investigate the role of dataset distances such as OTDD and OTD$^3$ in guiding data acquisition for PtO tasks. Given an initial labeled dataset $\D_L$ and an unlabeled pool $\D_U$, we iteratively select samples to label by computing distances between the validation set and candidate acquisitions. At each acquisition round, we approximate labels using proxy supervision from the nearest neighbors in the labeled set and evaluate the impact of adding different subsets of $\D_U$ on the dataset distance. We then acquire the subset that minimizes the distance and update the labeled dataset accordingly. 

% \TODO{add results and complete}

\subsection{Robustness to Model Complexity}

\begin{figure*}
    \includegraphics[width=\textwidth]{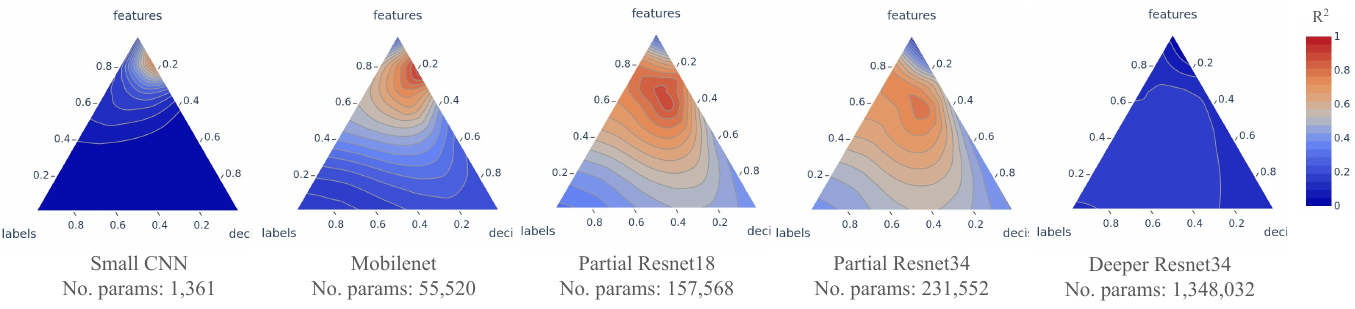}
    \caption{\textbf{OTD$^3$ Performance as Model Complexity Increases}. Predictability of OTD$^3$ (measured via R$^2$) across feature/label/decision weightings, shown for increasingly expressive model architectures on the Warcraft Shortest Path setting.
}
\label{fig:model_complexity}
\end{figure*}

We assess the robustness of OTD$^3$ by evaluating its performance across five model architectures of increasing complexity in the Warcraft setting (Figure~\ref{fig:model_complexity}). Although OTD$^3$ is model-agnostic by design, we measure its effectiveness through its ability to predict model transferability, specifically via $R^2$ values across weight configurations. In this setting, we find that intermediate-complexity models (Mobilenet, Partial ResNet18, Partial ResNet34) exhibit both high $R^2$ and broad regions of strong performance. This suggests that OTD$^3$ can reliably identify informative weightings when the underlying regret landscape is structured yet stable. The smoothness of these regions also indicates robustness to variations in component weights.

At the lower end of the complexity spectrum, the Small CNN displays low regret variance overall. Still, OTD$^3$ is able to highlight a narrow region of relative predictive strength—helping distinguish among otherwise uniformly weak configurations. In contrast, the most complex model (Deeper ResNet34) presents high regret variance, likely due to over-parameterization relative to the limited data. In this case, OTD$^3$ struggles to recover consistent patterns, reflecting the difficulty of transferability prediction in noisy, unstable regret landscapes.

These results suggest OTD$^3$ is most effective when regret variation is meaningful but not overly erratic, performing optimally when relationships between dataset distance and transferability are discernible and not obscured by noise from inappropriate model complexity. The robust performance across reasonably complex architectures highlights OTD$^3$'s practical utility in PtO scenarios.

%====================================================

\section{Discussion}
We introduce the first dataset distance tailored to PtO tasks, integrating features, labels, and decisions to better assess prediction-to-decision similarities. Our experiments show that incorporating decisions significantly improves transfer predictability, particularly in complex settings where label shifts do not directly correlate with decisions. This approach effectively captures task dynamics dictated by downstream optimization structures without requiring explicit analysis. Moreover, our framework is adaptable, allowing flexible weighting of components to provide meaningful comparisons across diverse PtO tasks—an essential feature for real-world applications where datasets vary not only in features and labels but also in decision complexity.

Several promising directions can extend our framework. Handling decision components of varying structures and dimensions using techniques like the Gromov-Wasserstein \citep{memoli_gromovwasserstein_2011} distance could bridge gaps between non-comparable decision spaces. Further refining the weighting of features, labels, and decisions—particularly through tuning methods independent of transferability measures—could enhance its utility. Additionally, adapting our approach to more intricate PtO structures, such as those where multiple feature-label pairs define a single decision, through a hierarchical OT framework \citep{yurochkin_hierarchical_2019}, could further improve its applicability.\looseness=-1

By establishing this first notion of dataset distance designed for PtO tasks, our work lays a foundation for future research, opening avenues for more robust and versatile transferability metrics in decision-aware learning.

\subsection*{Acknowledgments}

We thank Sanket Shah for insightful discussions throughout the development of this work. We also thank the reviewers and participants at the DMLR Workshop and the Humans, Algorithmic Decision-Making, and Society Workshop at ICML 2024 for their feedback on earlier versions, and the UAI reviewers for their useful feedback that significantly improved this work.\looseness=-1

PRD acknowledges support from the NSF under the AI Institute for Societal Decision Making (AI-SDM), Award No. 2229881. KW acknowledges support from NSF IIS-2403240 and the Schmidt Sciences AI2050 Fellowship. DAM acknowledges support from the Kempner Institute, the Aramont Fellowship Fund, and the FAS Dean’s Competitive Fund for Promising Scholarship.

%====================================================
\bibliographystyle{plainnat}
\bibliography{references.bib}

\begin{thebibliography}{37}
\providecommand{\natexlab}[1]{#1}
\providecommand{\url}[1]{\texttt{#1}}
\expandafter\ifx\csname urlstyle\endcsname\relax
  \providecommand{\doi}[1]{doi: #1}\else
  \providecommand{\doi}{doi: \begingroup \urlstyle{rm}\Url}\fi

\bibitem[Alvarez-Melis and Fusi(2020)]{alvarez-melis_geometric_2020}
David Alvarez-Melis and Nicolo Fusi.
\newblock Geometric {Dataset} {Distances} via {Optimal} {Transport}.
\newblock In \emph{Proceedings of the 34th {Conference} on {Neural} {Information} {Processing} {Systems}}, 2020.

\bibitem[Alvarez-Melis et~al.(2018)Alvarez-Melis, Jaakkola, and Jegelka]{alvarez-melis_structured_2018}
David Alvarez-Melis, Tommi Jaakkola, and Stefanie Jegelka.
\newblock Structured {Optimal} {Transport}.
\newblock In \emph{Proceedings of the {Twenty}-{First} {International} {Conference} on {Artificial} {Intelligence} and {Statistics}}. PMLR, 2018.

\bibitem[Amos and Kolter(2017)]{amos_optnet_2017}
Brandon Amos and J.~Zico Kolter.
\newblock {OptNet}: {Differentiable} {Optimization} as a {Layer} in {Neural} {Networks}.
\newblock In \emph{Proceedings of the 34th {International} {Conference} on {Machine} {Learning}}, 2017.

\bibitem[Bansal et~al.(2023)Bansal, Chen, Mukadam, and Amos]{bansal_taskmet_2023}
Dishank Bansal, Ricky T.~Q. Chen, Mustafa Mukadam, and Brandon Amos.
\newblock {TaskMet}: {Task}-{Driven} {Metric} {Learning} for {Model} {Learning}.
\newblock In \emph{Proceedings of the 37th {International} {Conference} on {Neural} {Information} {Processing} {Systems}}, 2023.

\bibitem[Ben-David et~al.(2010)Ben-David, Blitzer, Crammer, Kulesza, Pereira, and Vaughan]{ben-david_theory_2010}
Shai Ben-David, John Blitzer, Koby Crammer, Alex Kulesza, Fernando Pereira, and Jennifer~Wortman Vaughan.
\newblock A {Theory} of {Learning} from {Different} {Domains}.
\newblock \emph{Machine Learning}, 79:\penalty0 151--175, 2010.

\bibitem[Courty et~al.(2014)Courty, Flamary, and Tuia]{courty_domain_2014}
Nicolas Courty, Rémi Flamary, and Devis Tuia.
\newblock Domain {Adaptation} with {Regularized} {Optimal} {Transport}.
\newblock In \emph{Machine {Learning} and {Knowledge} {Discovery} in {Databases}}, pages 274--289, 2014.

\bibitem[Courty et~al.(2017)Courty, Flamary, Habrard, and Rakotomamonjy]{courty_joint_2017}
Nicolas Courty, Rémi Flamary, Amaury Habrard, and Alain Rakotomamonjy.
\newblock Joint {Distribution} {Optimal} {Transportation} for {Domain} {Adaptation}.
\newblock In \emph{Proceedings of the 31st {Conference} on {Neural} {Information} {Processing} {Systems}}, 2017.

\bibitem[Delon and Desolneux(2020)]{delon_wasserstein-type_2020}
Julie Delon and Agnès Desolneux.
\newblock A {Wasserstein}-{Type} {Distance} in the {Space} of {Gaussian} {Mixture} {Models}.
\newblock \emph{SIAM Journal on Imaging Sciences}, 13\penalty0 (2):\penalty0 936--970, 2020.

\bibitem[Donti et~al.(2017)Donti, Amos, and Kolter]{donti_task-based_2017}
Priya Donti, Brandon Amos, and J.~Zico Kolter.
\newblock Task-based {End}-to-end {Model} {Learning} in {Stochastic} {Optimization}.
\newblock In \emph{Proceedings of the 31st {Conference} on {Neural} {Information} {Processing} {Systems}}, 2017.

\bibitem[Elmachtoub and Grigas(2022)]{elmachtoub_smart_2022}
Adam~N. Elmachtoub and Paul Grigas.
\newblock Smart “{Predict}, then {Optimize}”.
\newblock \emph{Management Science}, 68\penalty0 (1):\penalty0 9--26, 2022.

\bibitem[Elmachtoub et~al.(2025)Elmachtoub, Lam, Zhang, and Zhao]{elmachtoub_estimate-then-optimize_2025}
Adam~N. Elmachtoub, Henry Lam, Haofeng Zhang, and Yunfan Zhao.
\newblock Estimate-{Then}-{Optimize} versus {Integrated}-{Estimation}-{Optimization} versus {Sample} {Average} {Approximation}: {A} {Stochastic} {Dominance} {Perspective}, 2025.

\bibitem[Frogner et~al.(2019)Frogner, Mirzazadeh, and Solomon]{frogner_learning_2019}
Charlie Frogner, Farzaneh Mirzazadeh, and Justin Solomon.
\newblock Learning {Embeddings} into {Entropic} {Wasserstein} {Spaces}.
\newblock In \emph{International {Conference} on {Learning} {Representations}}, 2019.

\bibitem[Janati et~al.(2019)Janati, Cuturi, and Gramfort]{janati_wasserstein_2019}
Hicham Janati, Marco Cuturi, and Alexandre Gramfort.
\newblock Wasserstein {Regularization} for {Sparse} {Multi}-{Task} {Regression}.
\newblock In \emph{Proceedings of the {Twenty}-{Second} {International} {Conference} on {Artificial} {Intelligence} and {Statistics}}, 2019.

\bibitem[Jiang et~al.(2023)Jiang, Zou, Liang, and Kwon]{jiang_opendataval_2023}
Kevin~Fu Jiang, James Zou, Weixin Liang, and Yongchan Kwon.
\newblock {OpenDataVal}: a {Unified} {Benchmark} for {Data} {Valuation}.
\newblock In \emph{Proceedings of the 37th {Conference} on {Neural} {Information} {Processing} {Systems}}, 2023.

\bibitem[Johnson-Yu et~al.(2023)Johnson-Yu, Finocchiaro, Wang, Vorobeychik, Sinha, Taneja, and Tambe]{johnson-yu_characterizing_2023}
Sonja Johnson-Yu, Jessie Finocchiaro, Kai Wang, Yevgeniy Vorobeychik, Arunesh Sinha, Aparna Taneja, and Milind Tambe.
\newblock Characterizing and {Improving} the {Robustness} of {Predict}-{Then}-{Optimize} {Frameworks}.
\newblock In \emph{Decision and {Game} {Theory} for {Security}}, 2023.

\bibitem[Just et~al.(2023)Just, Kang, Wang, Zeng, Ko, Jin, and Jia]{just_lava_2023}
Hoang~Anh Just, Feiyang Kang, Jiachen~T. Wang, Yi~Zeng, Myeongseob Ko, Ming Jin, and Ruoxi Jia.
\newblock {LAVA}: {Data} {Valuation} without {Pre}-{Specified} {Learning} {Algorithms}.
\newblock In \emph{International {Conference} on {Learning} {Representations}}, 2023.

\bibitem[Kantorovitch(1942)]{kantorovitch_translocation_1942}
L.~Kantorovitch.
\newblock On the {Translocation} of {Masses}.
\newblock \emph{Dokl. Akad. Nauk SSSR}, 37\penalty0 (1):\penalty0 227--229, 1942.
\newblock ISSN 0002-3264.

\bibitem[Liu et~al.(2025)Liu, Grigas, Liu, and Shen]{liu_active_2025}
Mo~Liu, Paul Grigas, Heyuan Liu, and Zuo-Jun~Max Shen.
\newblock Active {Learning} in the {Predict}-then-{Optimize} {Framework}: {A} {Margin}-{Based} {Approach}, 2025.
\newblock arXiv:2305.06584.

\bibitem[Mandi et~al.(2023)Mandi, Kotary, Berden, Mulamba, Bucarey, Guns, and Fioretto]{mandi_decision-focused_2023}
Jayanta Mandi, James Kotary, Senne Berden, Maxime Mulamba, Victor Bucarey, Tias Guns, and Ferdinando Fioretto.
\newblock Decision-{Focused} {Learning}: {Foundations}, {State} of the {Art}, {Benchmark} and {Future} {Opportunities}, 2023.

\bibitem[Mansour et~al.(2009)Mansour, Mohri, and Rostamizadeh]{mansour_domain_2009}
Yishay Mansour, Mehryar Mohri, and Afshin Rostamizadeh.
\newblock Domain {Adaptation}: {Learning} {Bounds} and {Algorithms}.
\newblock In \emph{Proceedings of {The} 22nd {Annual} {Conference} on {Learning} {Theory}}, 2009.

\bibitem[Mercioni and Holban(2019)]{mercioni_survey_2019}
Marina~Adriana Mercioni and Stefan Holban.
\newblock A {Survey} of {Distance} {Metrics} in {Clustering} {Data} {Mining} {Techniques}.
\newblock In \emph{Proceedings of the 3rd {International} {Conference} on {Graphics} and {Signal} {Processing}}, 2019.

\bibitem[Muzellec and Cuturi(2018)]{muzellec_generalizing_2018}
Boris Muzellec and Marco Cuturi.
\newblock Generalizing {Point} {Embeddings} using the {Wasserstein} {Space} of {Elliptical} {Distributions}.
\newblock In \emph{Proceedings of the 32nd {Conference} on {Neural} {Information} {Processing} {Systems}}, 2018.

\bibitem[Mémoli(2011)]{memoli_gromovwasserstein_2011}
Facundo Mémoli.
\newblock Gromov–{Wasserstein} {Distances} and the {Metric} {Approach} to {Object} {Matching}.
\newblock \emph{Foundations of Computational Mathematics}, 11\penalty0 (4):\penalty0 417--487, 2011.
\newblock ISSN 1615-3383.

\bibitem[Müller(1997)]{muller_integral_1997}
Alfred Müller.
\newblock Integral {Probability} {Metrics} and {Their} {Generating} {Classes} of {Functions}.
\newblock \emph{Advances in Applied Probability}, 29\penalty0 (2):\penalty0 429--443, 1997.
\newblock ISSN 0001-8678.

\bibitem[Qi et~al.(2023)Qi, Grigas, and Shen]{qi_integrated_2023}
Meng Qi, Paul Grigas, and Zuo-Jun~Max Shen.
\newblock Integrated {Conditional} {Estimation}-{Optimization}, 2023.
\newblock arXiv:2110.12351.

\bibitem[Ren et~al.(2024)Ren, Byun, and Wilder]{ren_decision-focused_2024}
Kevin Ren, Yewon Byun, and Bryan Wilder.
\newblock Decision-{Focused} {Evaluation} of {Worst}-{Case} {Distribution} {Shift}.
\newblock In \emph{Proceedings of the {Fortieth} {Conference} on {Uncertainty} in {Artificial} {Intelligence}}, pages 3076--3093. PMLR, 2024.

\bibitem[Shah et~al.(2022)Shah, Wang, Wilder, Perrault, and Tambe]{shah_decision-focused_2022}
Sanket Shah, Kai Wang, Bryan Wilder, Andrew Perrault, and Milind Tambe.
\newblock Decision-{Focused} {Learning} without {Differentiable} {Optimization}: {Learning} {Locally} {Optimized} {Decision} {Losses}.
\newblock In \emph{Proceedings of the 36th {Conference} on {Neural} {Information} {Processing} {Systems}}, 2022.

\bibitem[Shah et~al.(2023)Shah, Perrault, Wilder, and Tambe]{shah_leaving_2023}
Sanket Shah, Andrew Perrault, Bryan Wilder, and Milind Tambe.
\newblock Leaving the {Nest}: {Going} {Beyond} {Local} {Loss} {Functions} for {Predict}-{Then}-{Optimize}.
\newblock In \emph{Proceedings of the {AAAI} {Conference} on {Artificial} {Intelligence}}, 2023.

\bibitem[Shui et~al.(2019)Shui, Abbasi, Robitaille, Wang, and Gagné]{shui_principled_2019}
Changjian Shui, Mahdieh Abbasi, Louis-Émile Robitaille, Boyu Wang, and Christian Gagné.
\newblock A {Principled} {Approach} for {Learning} {Task} {Similarity} in {Multitask} {Learning}.
\newblock In \emph{Proceedings of the {Twenty}-{Eighth} {International} {Joint} {Conference} on {Artificial} {Intelligence}}, 2019.

\bibitem[Tran et~al.(2019)Tran, Nguyen, and Hassner]{tran_transferability_2019}
Anh Tran, Cuong Nguyen, and Tal Hassner.
\newblock Transferability and {Hardness} of {Supervised} {Classification} {Tasks}.
\newblock In \emph{Proceedings of the {IEEE}/{CVF} {International} {Conference} on {Computer} {Vision} ({ICCV})}, 2019.

\bibitem[Verdú(2014)]{verdu_total_2014}
Sergio Verdú.
\newblock Total {Variation} {Distance} and the {Distribution} of {Relative} {Information}.
\newblock In \emph{Information {Theory} and {Applications} {Workshop} ({ITA})}, pages 1--3, 2014.

\bibitem[Villani(2008)]{villani_optimal_2008}
Cédric Villani.
\newblock \emph{Optimal {Transport}, {Old} and {New}}, volume 338.
\newblock Springer Science \& Business Media, Berlin, Heidelberg, 2008.
\newblock ISBN 978-3-540-71049-3.

\bibitem[Vlastelica et~al.(2020)Vlastelica, Paulus, Musil, Martius, and Rolinek]{vlastelica_differentiation_2020}
Marin Vlastelica, Anselm Paulus, Vit Musil, Georg Martius, and Michal Rolinek.
\newblock Differentiation of {Blackbox} {Combinatorial} {Solvers}.
\newblock In \emph{Eighth {International} {Conference} on {Learning} {Representations}}, 2020.

\bibitem[Wang et~al.(2020)Wang, Wilder, Perrault, and Tambe]{wang_automatically_2020}
Kai Wang, Bryan Wilder, Andrew Perrault, and Milind Tambe.
\newblock Automatically {Learning} {Compact} {Quality}-aware {Surrogates} for {Optimization} {Problems}.
\newblock In \emph{Proceedings of the 34th {Conference} on {Neural} {Information} {Processing} {Systems}}, 2020.

\bibitem[Wilder et~al.(2019)Wilder, Dilkina, and Tambe]{wilder_melding_2019}
Bryan Wilder, Bistra Dilkina, and Milind Tambe.
\newblock Melding the {Data}-{Decisions} {Pipeline}: {Decision}-{Focused} {Learning} for {Combinatorial} {Optimization}.
\newblock In \emph{Proceedings of the {Thirty}-{Third} {AAAI} {Conference} on {Artificial} {Intelligence}}, 2019.

\bibitem[Xie et~al.(2020)Xie, Dai, Chen, Dai, Zhao, Zha, Wei, and Pfister]{xie_differentiable_2020}
Yujia Xie, Hanjun Dai, Minshuo Chen, Bo~Dai, Tuo Zhao, Hongyuan Zha, Wei Wei, and Tomas Pfister.
\newblock Differentiable {Top}-k {Operator} with {Optimal} {Transport}.
\newblock In \emph{Proceedings of the 34th {Conference} on {Neural} {Information} {Processing} {Systems}}, 2020.

\bibitem[Yurochkin et~al.(2019)Yurochkin, Claici, Chien, Mirzazadeh, and Solomon]{yurochkin_hierarchical_2019}
Mikhail Yurochkin, Sebastian Claici, Edward Chien, Farzaneh Mirzazadeh, and Justin~M Solomon.
\newblock Hierarchical {Optimal} {Transport} for {Document} {Representation}.
\newblock In \emph{Proceedings of the 33rd {Conference} on {Neural} {Information} {Processing} {Systems}}, 2019.

\end{thebibliography}
%====================================================

\newpage
\onecolumn
\title{What is the Right Notion of Distance between Predict-then-Optimize Tasks?\\(Supplementary Material)}
\maketitle

\appendix
%=========================
\section{Proof of proposition~\ref{prop:metric}}
%=========================

To demonstrate that the OTD$^3$, $d_{OT}(\cdot, \cdot; c_{PtO})$ is a valid metric, it is sufficient to verify that the ground cost function $c_{PtO}$ used in the optimal transport problem is a metric on $\X \times \Y \times \Omega$. If $c_{PtO}$ is indeed a metric, then $d_{OT}(\cdot, \cdot; c_{PtO})$ corresponds to the Wasserstein distance \cite{villani_optimal_2008}. In Equation~\ref{eq:cpto}, $d_{OT}(\cdot, \cdot; c_{PtO})$ is defined as a convex combination of $d_{\X}$ and $d_{\Y}$, which are metrics on $\X$ and $\Y$ respectively, and the decision quality disparity $l_q$. To show that $c_{PtO}$ is a metric, it suffices to show that $l_q$ satisfies the four metric properties: non-negativity, identity of indiscernibles, symmetry, and the triangle inequality. If $l_q$ does not individually satisfy these properties, we must demonstrate that the convex combination of $d_{\X}$, $d_{\Y}$, and $l_q$ satisfies these properties collectively under the assumption that $\alpha_X,\alpha_Y,\alpha_W > 0$.

First, $l_q$ is clearly non-negative because it is defined as an absolute value. It is symmetric in the convex combination of $c_{PtO}$ because it is taken as the absolute difference between two decision qualities with fixed true costs. 
\begin{align*}
    \ldq(z, z'; y', y') 
    &= \big| q(z; y') - q(z'; z') \big| \\
    &= \big| q(z'; y') - q(z; z') \big| \\
    &= \ldq(z', z; y', y')
\end{align*}
Moreover, $l_q$ satisfies triangle inequality due to the triangle inequality property of the absolute value.
\begin{align*}
    &\ldq(z_1, z_2; y_1, y_2) + \ldq(z_2, z_3; y_2, y_3) \\
    &= \big| g(z_1; y_1) - g(z_2; y_2) \big| + \big| g(z_2; y_2) - g(z_3; y_3) \big| \\
    &\leq \big| g(z_1; y_1) - g(z_2; y_2) + g(z_2; y_2) - g(z_3; y_3) \big| \\
    &= \big| g(z_1; y_1)  - g(z_3; y_3) \big| \\
    &= \ldq(z_1, z_3; y_1, y_3)
\end{align*}
Lastly, while $l_q$ might not satisfy the identity of indiscernibles in isolation (specifically, $l_q(y, y'; z, z) = 0$ does not necessarily imply $y = y'$; meaning two different decisions can lead to the same objective value), $c_{PtO}$ does satisfy this property for $\alpha_X,\alpha_Y,\alpha_W > 0$. If $(x,y,z) = (x',y',z')$, then $\ldq(z, z'; y', y') = \big| g(z; y) - g(z'; y) \big| = 0$ because $z=z'$ implies $g(z; y) = g(z'; y)$ and hence $c_{PtO}((x, y, z), (x', y', z'))=0$. Conversely, if $c_{PtO}((x, y, z), (x', y', z')) = 0$, then $d_{\X}(x, x') = 0$, $d_{\Y}(y, y') = 0$, and $l_q(y, y'; z, z) = 0$ because $\alpha_X,\alpha_Y,\alpha_W > 0$. Since $d_{\Y}(y, y') = 0$ implies $y = y'$ (because $d_{\Y}$ is a metric), it follows that $w^*(y) = w^*(y')$ and hence $z = z'$.

Therefore, $c_{PtO}$ satisfies the identity of indiscernibles. Consequently, since $l_q$ satisfies non-negativity, symmetry, and the triangle inequality, and since $c_{PtO}$ satisfies the identity of indiscernibles, $d_{OT}(\cdot, \cdot; c_{PtO})$ is indeed a valid metric with $c_{PtO}$ a valid metric on $\X \times \Y \times \Omega$.

\section{Preamble for Theorem~\ref{theo1}}

\subsection{Validity Assumption \ref{assump:1}}
Assumption~\ref{assump:1} imposes a specific structure on the downstream optimization problem by assuming that the decision quality function has a bounded rate of change with respect to both the predicted and true cost vectors. This is a reasonable assumption for certain downstream optimization tasks, as highlighted in the following lemmas. 

\begin{lemma}
\label{lemma1}
    If $M(\cdot)$ is a convex program with a strongly convex objective and constraints with independent derivatives (Linear Independence Constraint Qualification (LICQ)), Assumption~\ref{assump:1} holds. 
\end{lemma}
The strong convexity of the objective ensures that the gradient is Lipschitz continuous, while the LICQ guarantees that the optimal solutions depend continuously on the parameters. By the smoothness of the objective and the continuity of the optimal solutions, the difference in the decision quality function $q$ between two sets of parameters and their corresponding optimal solutions can be bounded by a linear combination of the distances between the parameters and the distances between the optimal solutions.

\begin{lemma}
\label{lemma2}
    If $M(\cdot)$ has a linear optimization objective with a strongly convex feasible region, Assumption \ref{assump:1} holds.
\end{lemma}
When $M(\cdot)$ has a linear optimization objective and a strongly convex feasible region, the decision quality function $q$ satisfies the $k_1,k_2$-Lipschitz property. The linearity of the objective ensures that changes in the parameters lead to proportional changes in the objective value, while the strong convexity of the feasible region guarantees that the optimal solutions are unique and vary smoothly with respect to the parameters. This smooth dependence, combined with the linear structure of the objective, implies that the difference in $q$ between two sets of parameters and their corresponding optimal solutions can be bounded by a linear combination of the distances between the parameters and the distances between the optimal solutions.

\subsection{Lipschitzness of the Decision Quality Disparity Function}
To establish the bound presented in Theorem \ref{theo1}, we rely on the fact that $\ldq$ is $k_1,k_2$-Lipschitz under Assumption \ref{assump:1}. The following proposition demonstrates that $\ldq$ indeed satisfies the Lipschitz condition given this assumption.

\begin{proposition}
If $g$, the objective function of the downstream optimization problem, is $k_1,k_2$-Lipschitz (Assumption \ref{assump:1}), then $\ldq$ is also $k_1,k_2$-Lipschitz.
\end{proposition}

\begin{proof}
    \begin{align}
    &\big| \ldq(z,z_1; y,y_1) - \ldq(z,z_2; y,y_2) \big| \notag \\
    &= \big| \abs{g(z;y)-g(z_1;y_1)} - \abs{g(z;y)-g(z_2;y_2)} \big| \notag \\
    &\leq \big| g(z;y) - g(z_1;y_1) - g(z;y) + g(z_2;y_2) \big| \label{eq:kl1}\\
    &= \big|g(z_2;y_2) - g(z_1;y_1) \big| \notag \\
    &= \big|g(z_2;y_2) - g(z_1;y_2) + g(z_1;y_2) - g(z_1;y_1)  \big| \notag \\
    &\leq \big|g(z_2;y_2) - g(z_1;y_2)\big| + \big|g(z_1;y_2) - g(z_1;y_1) \big| \label{eq:kl2}\\
    &\leq k_1  \norm{z_1-z_2} + k_2 \norm{y_1-y_2} \label{eq:kl3}
\end{align}
Inequalities (\ref{eq:kl1}) and (\ref{eq:kl2}) are a result of the triangle inequality of the absolute value. Inequality (\ref{eq:kl3}) is due to the $k_1-k_2$-lipschitzness of $g$.
\end{proof}

\section{Proof of Theorem~\ref{theo1}} \label{sec:proof_theorem}

\begin{theorem}    
    Suppose Assumption \ref{assump:1} holds. For a feature space $\X$, a label space $\Y$, and a decision set $\Omega$, let $\Z := \X \times \Y \times \Omega$. Let $\Pd_{S}$ and $\Pd_{T}$ be the source and target distributions over $\X \times \Y$ respectively. For any labeling function $f:\X \to \Y$, let $\Pd_{T}^{f}$ and $\Pd_{S}^{*}$ be distributions over $\Z$ given by $\Pd_{T}^{f} := (x,y,w^*(f(x)))_{(x,y) \sim \Pd_T}$ and $\Pd_{S}^{*} := (x,y,w^*(y))_{(x,y) \sim \Pd_S}$. For a ground cost function of the form \looseness=-1
    \begin{align*}
        c_{PtO}^{\bm{\alpha}}((x,y,z),(x',y',z')) = \alpha_X d_{\X}(x,x') + \alpha_Y d_{\Y}(y,y') + \alpha_W \ldq(z,z';y',y'), \notag
    \end{align*}
    let $\Pi^*$ be the coupling that minimizes the OT problem with ground cost $c_{PtO}^{\bm{\alpha}}$ between $\Pd_{T}^{f}$ and $\Pd_{S}^{*}$. Let $\Tilde{f}$ be a labeling function that is $\phi$-Lipschitz transferable w.r.t. $\Pi^*$. We assume $\X$ is bounded by $K$ and $\Tilde{f}$ is $l$-Lipschitz, such that $|\Tilde{f}(x_1)-\Tilde{f}(x_2)| \leq 2lK = L$. Then, for all $\lambda > 0$ and $\alpha_W \in (0,1)$ such that $(\lambda k_1 + k_2 + 1)\alpha_W = 1$, and $\alpha_X = \lambda k_1 \alpha_W$ and $\alpha_Y = k_2\alpha_W$, we have with probability at least $1-\delta$ that:
    \begin{align*}
        err(f; \dqreg, \Pd_T)  \leq \ err(\Tilde{f}; \dqreg, \Pd_S) + err(\Tilde{f}; \dqreg, \Pd_T) + k_1L\phi(\lambda)
        + (1/\alpha_{W}) d_{OT}(\Pd_{T}^{f}, \Pd_{S}^{*} \ ; c_{PtO}^{\bm{\alpha}})
    \end{align*}
\end{theorem}
\begin{proof}
\begin{align}
    er&r(f; \dqreg, \Pd_T) \notag \\
    % &= \Et \dqreg(f(x),y) \notag \\
    &= \Et \ldq(w^*(f(x)),w^*(y);y,y) \notag \\
    &\leq \Et \ldq(w^*(f(x)),w^*(\Tilde{f}(x));y,y) + \Et \ldq(w^*(\Tilde{f}(x)),w^*(y);y,y) \label{eq:bound1}\\
    &=  \Et \ldq(w^*(\Tilde{f}(x)), w^*(f(x));y,y) + \Et \ldq(w^*(\Tilde{f}(x)),w^*(y);y,y) \label{eq:bound2}\\
    &=  \Etf \ldq(w^*(\Tilde{f}(x)), z;y,y) + \Et \ldq(w^*(\Tilde{f}(x)),w^*(y);y,y) \label{eq:bound3} \\
    &=  \Etf \ldq(w^*(\Tilde{f}(x)), z;y,y) - err(\Tilde{f}; \dqreg, \Pd_S) + err(\Tilde{f}; \dqreg, \Pd_S) + err(\Tilde{f}; \dqreg, \Pd_T) \notag \\
    &=  \Etf \ldq(w^*(\Tilde{f}(x)), z;y,y) - \Etstar \ldq(w^*(\Tilde{f}(x)), z;y,y) + err(\Tilde{f}; \dqreg, \Pd_S) + err(\Tilde{f}; \dqreg, \Pd_T) \notag \\
    &\leq  \big| \Etf \ldq(w^*(\Tilde{f}(x)), z;y,y) - \Etstar \ldq(w^*(\Tilde{f}(x)), z;y,y)\big| + err(\Tilde{f}; \dqreg, \Pd_S) + err(\Tilde{f}; \dqreg, \Pd_T) \notag
\end{align}

Inequality (\ref{eq:bound1}) uses the fact that $\ldq(\ \cdot \ ;y,y)$ satisfies the triangle inequality and line (\ref{eq:bound2}) is due to the symmetry of $\ldq(\ \cdot \ ;y,y)$ for any $y \in \C$. Line (\ref{eq:bound3}) comes from the fact that $\Pd_{T}^{f} := (x,f(x),y)_{(x,y) \sim \Pd_T}$. We continue by bounding the first term.
\begin{align}
    &\big|\Etf \ldq(w^*(\Tilde{f}(x)), z;y,y) - \Etstar \ldq(w^*(\Tilde{f}(x)), z;y,y)\big| \notag \\[0.1cm]
    &=  \abs{\int_{\Z} \ldq(w^*(\Tilde{f}(x)),z;y,y)(\Pd_{T}^{f}(X=x, Y=y, Z=z)-\Pd_{S}^{*}(X=x, Y=y, Z=z))\dt x \dt y \dt z} \notag \\
    &=  \abs{\int_{\Z} \ldq(w^*(\Tilde{f}(x)),z;y,y) \dtproof}
        \notag \\
    &\leq  \int_{\Z^2} 
        \abs{\ldq(\Tilde{z}_t,z_{t}^{f};y_t,y_t)
        -\ldq(\Tilde{z}_s,z_{s};y_s,y_s)}
        \dtproofshort
        \label{eq:bound4}\\ %tag
    &\leq   \int_{\Z^2} 
            \Big|\ldq(\Tilde{z}_t,z_{t}^{f};y_t,y_t) 
            - \ldq(\Tilde{z}_s,z_{t}^{f};y_s,y_t)\Big|
            + \Big|\ldq(\Tilde{z}_s,z_{t}^{f};y_s,y_t) 
            - \ldq(\Tilde{z}_s,z_{s},;y_s,y_s)\Big| 
            \dtproofshort
            \label{eq:bound5}\\   
    &\leq   \int_{\Z^2} 
            k_1 d_\C(\Tilde{f}(x_t),\Tilde{f}(x_s)) + k_2 d_\C({y_{t},y_s})
            + \Big|\ldq(\Tilde{z}_s,z_{t}^{f};y_s,y_t) 
            - \ldq(\Tilde{z}_s,y_{s};y_s,y_s)\Big| 
            \dtproofshort
            \label{eq:bound6}\\
    &\leq   \ k_1L\phi(\lambda) + \int_{\Z^2} 
            \lambda k_1d_\mathcal{X}(x_t,x_s) 
            + k_2 d_\C({y_{t},y_s})
            + \Big|\ldq(\Tilde{z}_s,z_{t}^{f};y_s,y_t) - \ldq(\Tilde{z}_s,y_{s};y_s,y_s)\Big| 
            \dtproofshort\\
    &\leq   \ k_1L\phi(\lambda) + \int_{\Z^2} 
            \lambda k_1 d_\mathcal{X}(x_t,x_s) 
            + k_2 d_\C({y_{t},y_s}) + \ldq(z_{t}^{f},z_s;y_s,y_s) 
            \dtproofshort
            \notag
\end{align}

From line (\ref{eq:bound4}) onwards we take $\vb{w}_{s}:=(x_s,y_s,y_s), \vb{w}_{t}^{f}:=(x_t,y_{t}^{f},y_t)$ and $\Tilde{z}_s = w^*(\Tilde{f}(x_s)), \Tilde{z}_t = w^*(\Tilde{f}(x_t))$ for ease of notation. Given a weight $\alpha_W$, we now normalize the last term such that the ground cost function is a convex combination of $d_\X$, $d_\Y$m and $\ldq$.

\begin{align}
    & \int_{\Z^2} \lambda k_1 d_\mathcal{X}(x_t,x_s) 
    + k_2 d_\C({y_{t},y_s}) + \ldq(z_{t}^{f},z_s;y_s,y_s) 
    \dtproofshort
    \notag \\
    &= \frac{1}{\alpha_W}\int_{\Z^2} 
    \lambda k_1 \alpha_W  d_\X(x_t,x_s) 
    + k_2 \alpha_W d_\C(x_t,x_s)
    + \alpha_W  \ldq(z_{t}^{f},z_s;y_s,y_s) 
    \dtproofshort
    \notag \\
    &= \frac{1}{\alpha_W} \ d_{OT}(\Pd^{f}_{T},\Pd^{*}_{S}; c_{PtO}^{\bm{\alpha}}) \notag
\end{align}
\end{proof}

\section{Experimental Settings Details} \label{app:experiments}
\subsection{Linear Model Top-K \cite{shah_decision-focused_2022}}
\paragraph{PtO task description.} The Linear Model Top-K setting is a learning task designed to evaluate decision-focused learning approaches in scenarios where the true relationship between features and outcomes is nonlinear, yet the model used for prediction is constrained to be linear. Specifically, the objective is to train a linear model to perform top–$K$ selection when the underlying data is generated by a cubic polynomial function. This controlled setup enables an assessment of how well decision-focused methods handle model misspecification. The predict-then-optimize (PtO) task in this setting is defined as follows:

\begin{itemize}[leftmargin=0.5cm]
    \item[] \textit{Predict:} Given the feature $x_n \sim \Pd_\X$, where $\Pd_\mathcal{X} = \text{Unif}[-1, 1]$, of a resource $n$, the prediction tasks consists of using a linear model to predict the corresponding utility $\hat{y}_n$, where the true utility $y_n = p(x_n)$ is a cubic polynomial in $x_n$. The predictions for $N$ resources are aggregated into a vector $\hat{\bm{y}} = [\hat{y}_1, \ldots, \hat{y}_N]$, where each element corresponds to the predicted utility of a resource.
    
    \item[] \textit{Optimize:} The optimization task involves selecting the $K$ out of $N$ resources with the highest utility. This corresponds to solving the optimization problem $M(\hat{\bm{y}}) = \max_{\bm{z} \in [0,1]^N} \{\bm{z} \cdot \sigma_x(\hat{\bm{y}})\}$ such that $||\bm{z}||_0 = K$, where $\sigma_x$ is the permutation that orders $\hat{\bm{y}}$ in ascending order of $\bm{x} = [x_1, \ldots, x_N]$.
\end{itemize}

\paragraph{Synthetic distribution shift} We introduce synthetic distribution shifts to create a scenario for transfer learning. We modify the original feature-label distribution \( \Pd = (\text{Id}, p)_*U[-1,1] \). Specifically, for various values of \( \gamma \in [0,1.3] \), we define the feature-label distributions \( \Pd_\gamma = (\text{Id}, p_\gamma)_*U[-1,1] \) where \( p_{\gamma}(x_n) = 10(x_{n}^3 - \gamma x_n) \), using \( \Pd_{0.65} \) as the target distribution.

\paragraph{Training details} We use the implementation from \citet{shah_decision-focused_2022}\footnote{\url{github.com/sanketkshah/LODLs}} to train models by setting \texttt{loss="DFL"}. This implementation uses an entropy regularized Top-K loss function proposed by \citet{xie_differentiable_2020} that reframes the Top-K problem with entropy regularization as an optimal transport problem, enabling end-to-end learning.

\subsection{Warcraft Shortest Path \cite{vlastelica_differentiation_2020}}
\paragraph{PtO task description.} This setting involves finding the minimum-cost path on $d \times d$ RGB grid maps from the Warcraft II tileset dataset, where each pixel represents terrain with an unknown traversal cost. The task is to first predict these costs from an input image and then determine the shortest path from the top-left to the bottom-right corner based on the predicted cost map. This benchmark is particularly notable because it involves image inputs, a modality not widely explored in other shortest-path learning tasks. Following \cite{vlastelica_differentiation_2020}, we use $96 \times 96$ RGB images as input, with the shortest path being computed on a coarser $12 \times 12$ grid representation of the predicted costs.

\begin{itemize}[leftmargin=0.5cm]
    \item[] \textit{Predict:} Given the feature $x_n \in \mathbb{R}^{d \times d \times 3}$, predict the travel cost grid $\hat{y}_n \in \mathbb{R}^{p \times p}$.
    \item[] \textit{Optimize:} Solve a shortest-path problem over the predicted cost grid. Specifically, find the path $\bm{z}$ that minimizes the total traversal cost: $M(\hat{\bm{y}}) = \min_{\bm{z} \in [0,1]^p} \{\bm{z} \cdot \hat{y}\}$ subject to boundary conditions $\bm{z}_{0,0}=\bm{z}_{p,p}=1$ and connectivity constraints ensuring that $\bm{z}$ represents a valid path from the top-left to the bottom-right corner.
\end{itemize}

\paragraph{Synthetic distribution shift.} The original distribution $\mathcal{P}$, which we treat as the target distribution, is defined over $\mathbb{R}^{d \times d} \times \mathbb{R}^{p \times p}$, where $d = 96$ and $p = 12$. Here, $\mathbb{R}^{d \times d}$ represents the feature space depicting maps, while $\mathbb{R}^{p \times p}$ represents the traveling costs on these maps. We induce a target shift for $\Pd_\gamma$ by uniformly sampling the costs for different pixel classes from the same range as  $\mathcal{P}$ ([0.8, 9.2] for the Warcraft II tileset dataset). Figure~\ref{fig:warcraft_shift} illustrates the costs coming from two different distributions over one same feature while highlighting the different decisions (shortest path) that these costs yield.

\paragraph{Training details.} We use \texttt{pyepo}\footnote{\url{github.com/khalil-research/PyEPO}} implementation with SPO+ loss function on a truncated ResNet-18 consisting of the first five layers, followed by a final convolutional layer that reduces the number of output channels to one. Finally, we use an adaptive max-pooling layer to obtain a fixed $p \times p$ spatial resolution, allowing for a structured representation of the extracted features.

\begin{figure}[h!]
    \centering
    \includegraphics[width=3.25in]{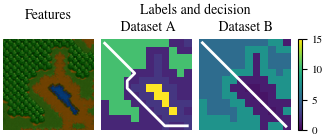}
    \caption{\textit{Synthetic distribution shift in Warcraft Shortest Path}. The white line illustrates the decision, corresponding to the shortest path, on dataset A (center) and dataset B (right) for a sample with the same features (left map).}
    \label{fig:warcraft_shift}
\end{figure}

\subsection{Inventory Stock Problem \cite{donti_task-based_2017}}
\paragraph{PtO task description.} In this problem a company must determine the optimal order quantity \(z\) of a product to minimize costs given a stochastic demand \(y\), which is influenced by observed features \(x\). The cost structure includes both linear and quadratic costs for the amount of product ordered, as well as different linear and quadratic costs for over-orders \([z - y]^+\) and under-orders \([y - z]^+\). The objective function is:
\begin{align}
f_{\text{stock}}(y, z) =& c_{0} z + \frac{1}{2} q_{0} z^{2} + c_{b} [y - z]_{+} + \frac{1}{2} q_{b} ([y - z]_{+})^{2} \notag \\&+ c_{h} [z - y]_{+} + \frac{1}{2} q_{h} ([z - y]_{+})^{2} 
\end{align}

\noindent where \([v]_{+} \equiv \max \{v, 0\}\). In our paper, we use $c_0=30, q_0=10, c_b=10, q_b=2, c_h=30, q_h=25$.
For a given probability model \(p(y| x; \theta)\), the proxy stochastic programming problem can be formulated as: $\underset{z}{\operatorname{minimize}} \quad \mathbf{E}_{y \sim p(y|x; \theta)} \left[ f_{\text{stock}}(y, z) \right]$.

To simplify the setting, we assume that the demands are discrete, taking on values \(d_{1}, \ldots, d_{k}\) with probabilities (conditional on \(x\)) \(\left(p_{\theta}\right)_{i} \equiv p\left(y = d_{i}|x; \theta\right)\). Thus, our stochastic programming problem can be succinctly expressed as a joint quadratic program:
\begin{align*}
&\underset{z \in \mathbb{R}, z_{b}, z_{h} \in \mathbb{R}^{k}}{\operatorname{minimize}} \Big\{
c_{0} z + \frac{1}{2} q_{0} z^{2} + \sum_{i = 1}^{k} \left(p_{\theta}\right)_{i} \big( c_{b} (z_{b})_{i} \tag{10}\\
& \hspace{2.2cm}+ \frac{1}{2} q_{b} (z_{b})_{i}^{2} + c_{h} (z_{h})_{i} + \frac{1}{2} q_{h} (z_{h})_{i}^{2} \big)\Big\} \\
&\text{subject to} \quad d - z \mathbf{1} \leq z_{b}, \quad z \mathbf{1} - d \leq z_{h}, \quad z, z_{h}, z_{b} \geq 0
\end{align*}

\paragraph{Synthetic distribution shift} We generate problem instances by randomly sampling \(x \in \mathbb{R}^n\) and then generating \(p(y| x; \theta)\) according to \(p(y|x; \theta) \propto \exp((\theta^T x)^2)\). We introduce distribution shifts for both \(x\) and \(y\). Specifically, \(x\) is sampled from a Gaussian distribution where the mean is sampled from \(U[-0.5, 0.5]\), and \(\theta\) is also sampled from a Gaussian distribution.

\paragraph{Training details} We use the implementation from \cite{donti_task-based_2017}\footnote{\url{github.com/locuslab/e2e-model-learning}} following their Inventory Stock Problem experiments. 

\section{OTD$^3$ Implementation Details}
Our implementation of the OTD$^3$ relies on the \texttt{POT}\footnote{\url{pythonot.github.io/}} package. The computation of dataset distance involves two main steps:
\begin{enumerate}
    \item \textbf{Computing Pairwise Pointwise Distances:}
    We first compute the pairwise distances between samples in the source and target datasets. This involves calculating distances separately for features, labels, and decisions, weighted according to the selected component weights $(\alpha_X,\alpha_Y,\alpha_W)$. Feature and label distances are computed using standard metric spaces (e.g., Euclidean or cosine distance), while decision distances are computed using decision quality disparity.
    \item \textbf{Solving the Optimal Transport Problem:}
    Given the computed pairwise distances, we compute the dataset distance using Earth Mover’s Distance (EMD) via \texttt{POT}'s \texttt{emd} solver. EMD finds the exact optimal transport plan, making it well-suited for capturing true correspondences between source and target datasets without introducing regularization bias. This approach was computationally feasible in our experiments due to the relatively small dataset sizes.
\end{enumerate}
Additionally, for experiments involving hyperparameter tuning, we evaluate multiple weight combinations on a predefined grid and select the setting that maximizes correlation with regret transferability.

\section{Additional Results}
\subsection{Selecting source datasets for transfer learning}
In Section~\ref{sec:pred_transfer} we analyzed the correlation between dataset distance and transferability in PtO. The plots presented in Figure \ref{app-fig:regression} show this correlation for the Linear Model TopK setting and the Inventory Stock problem under two weighting profiles: one where decision-related features are excluded (left) and one where they are included (right). In both settings, incorporating decisions into the distance metric leads to improved predictability of transfer performance. This effect is more pronounced in the Linear Model TopK task than in the Inventory Stock problem.

For these settings, we do not perform fine-tuning on the target dataset. Instead, we assess transferability in a zero-shot setting, where a model trained on the source dataset is directly applied to the target domain without further adaptation. This choice is motivated by the relative simplicity of the feature spaces involved, which enables a meaningful evaluation of dataset distances without introducing potential confounding effects from additional training steps. Accordingly, rather than plotting dataset distance against the relative drop in regret after fine-tuning, we plot it against $\mathcal{T}(S \to T) = (\text{reg}(\D_{S})-\text{reg}(\D_{T}))/\text{reg}(\D_{T})$,
where $\text{reg}(\D_{S})$ denotes the decision regret when applying the source-trained model to the target dataset, and $\text{reg}(\D_{S})$ is the regret of a model trained directly on the target.

\begin{figure}[h!]
    \centering
    \includegraphics[width=0.45\columnwidth]{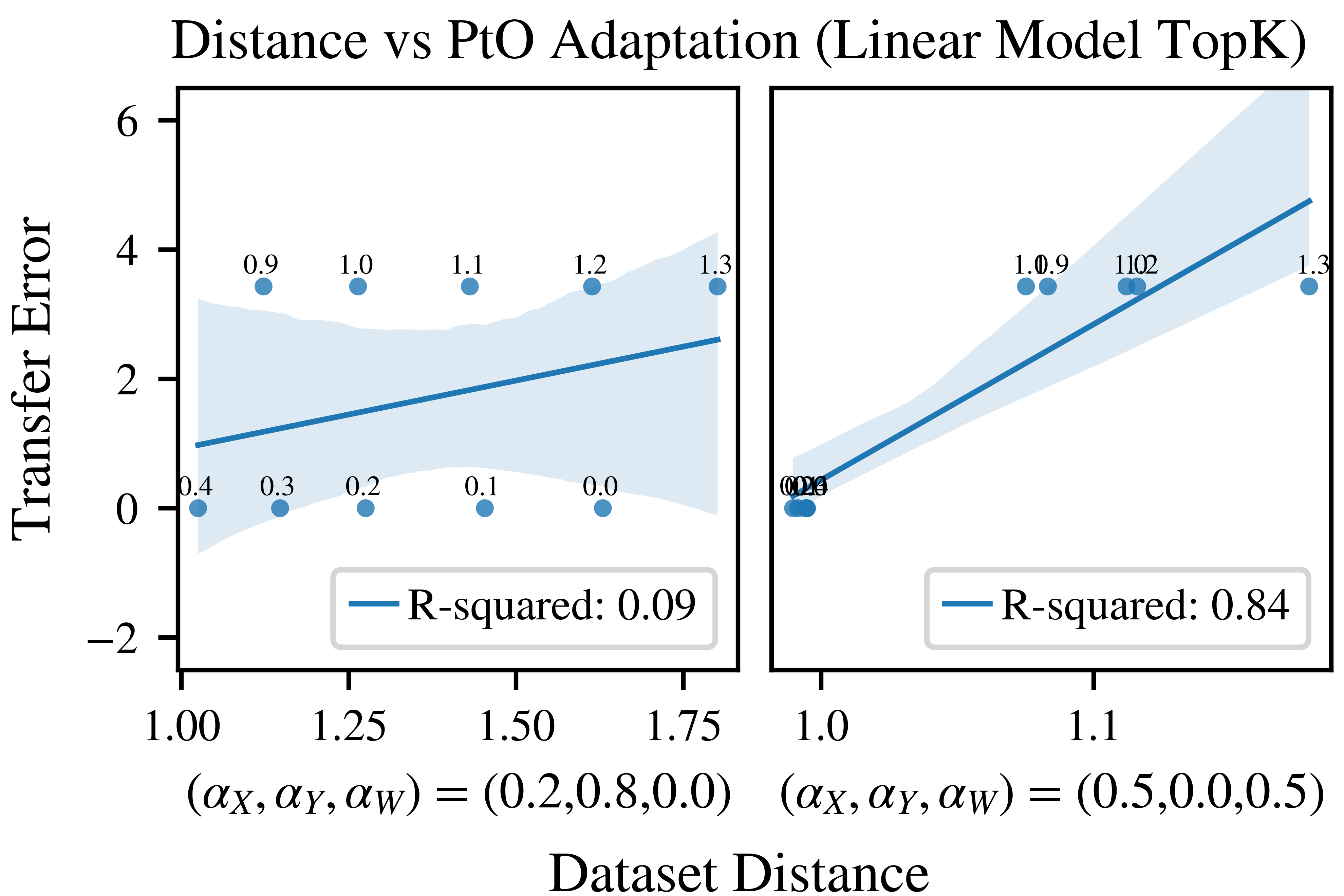}
    \includegraphics[width=0.45\columnwidth]{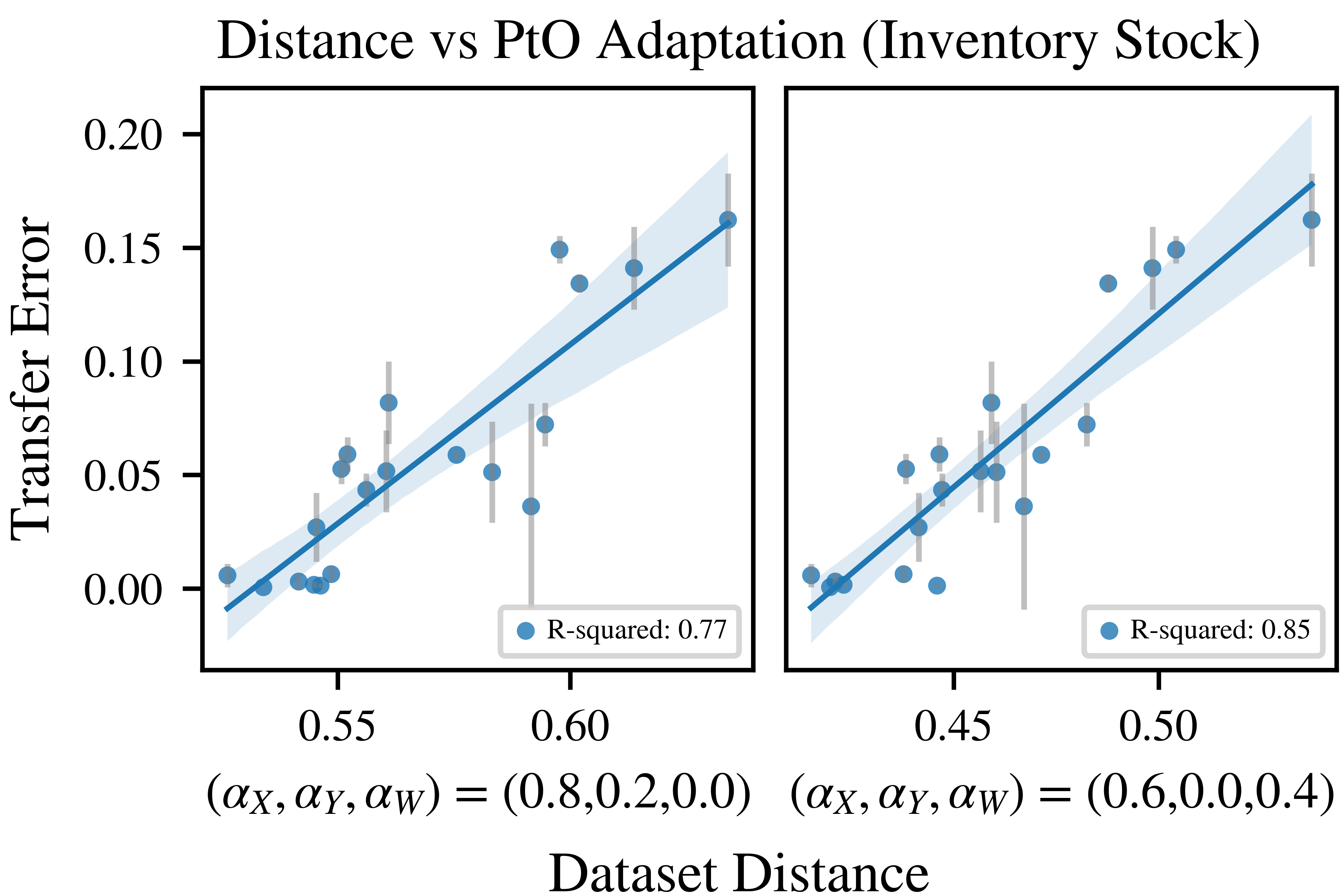}
    \caption{\textit{Distance vs Adaptation}. OT distance for the best feature-label and feature-label-decision weighting against regret transferability.}
    \label{app-fig:regression}
\end{figure}

\noindent
\begin{minipage}[t]{0.55\textwidth}
    \vspace{0pt} 
    \centering
    \begin{minipage}[b]{0.48\textwidth}
        \centering
        \includegraphics[height=1.6in]{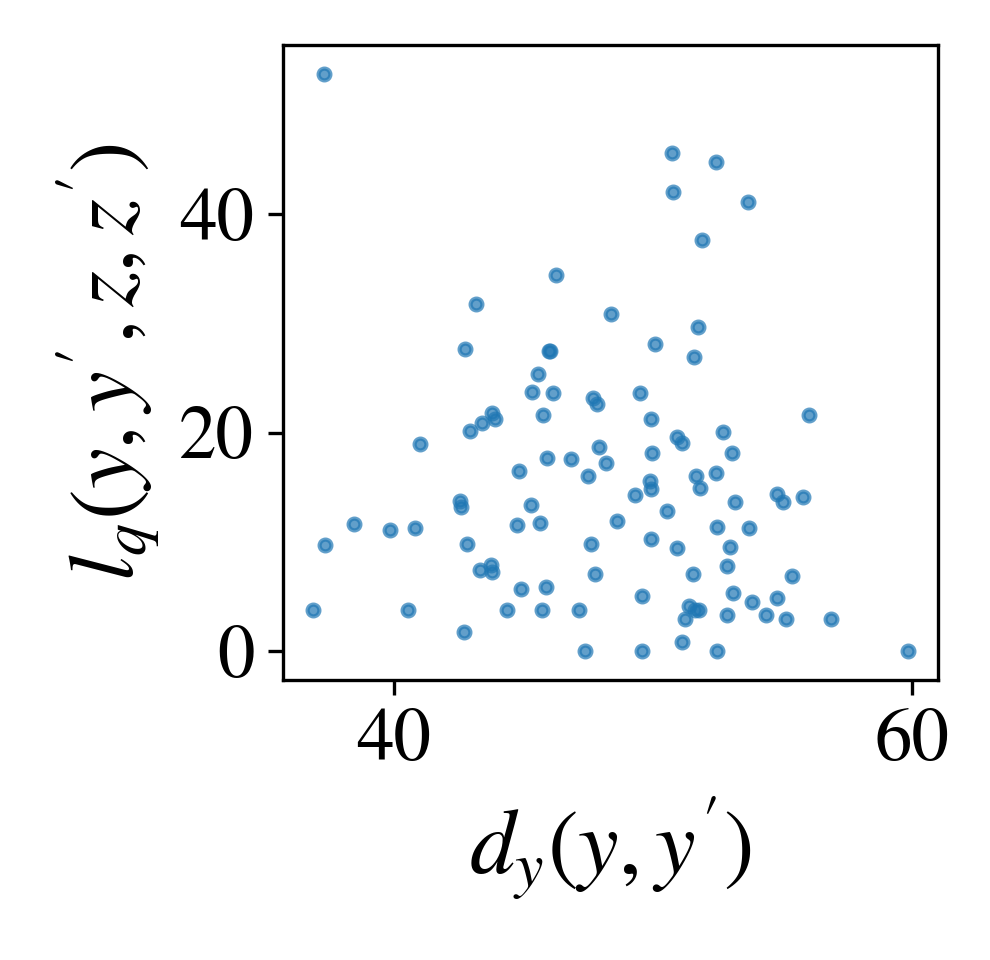}
        \makebox[\linewidth]{\small (a) Warcraft Shortest Path}
        \label{fig:correlation_warcraft}
    \end{minipage}%
    \hfill
    \begin{minipage}[b]{0.48\textwidth}
        \centering
        \includegraphics[height=1.6in]{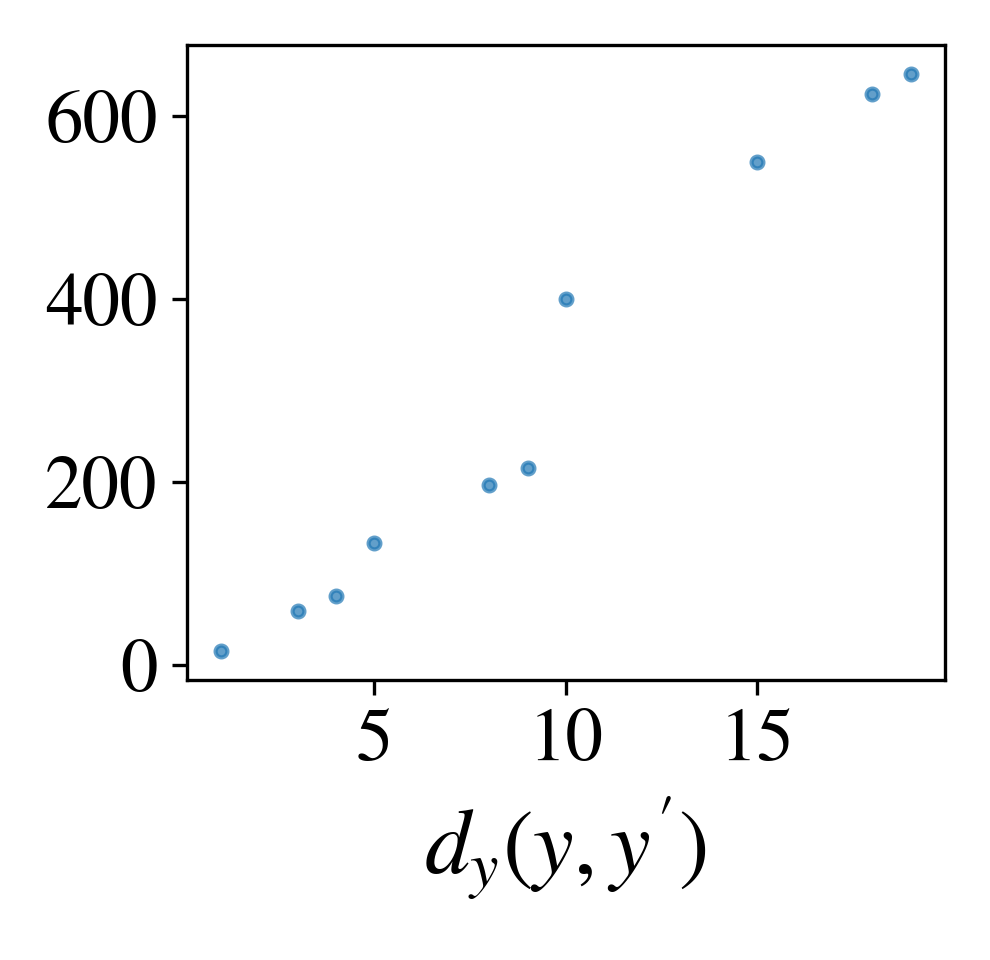}
        \makebox[\linewidth]{\small (b) Inventory Stock}
        \label{fig:correlation_inventory}
    \end{minipage}
    \captionof{figure}{Difference in labels against difference in decisions.}
\label{fig:combined_correlation}
\end{minipage}%
\hfill
\begin{minipage}[t]{0.4\textwidth}
    \vspace{0pt}
    To illustrate the relationship between label space differences $d_y(y,y')$ and decision space differences $l_q(y,y',z,z')$ in different PtO tasks, we provide the following visualizations. Figure~\ref{fig:correlation_inventory} shows this correlation for the Inventory Stock problem, while Figure~\ref{fig:correlation_warcraft} presents the same analysis for the Warcraft domain.
\end{minipage}

\end{document}